\documentclass{article} %
\PassOptionsToPackage{dvipsnames,table}{xcolor}
\usepackage[preprint]{colm2025_conference}

\usepackage{microtype}
\usepackage{epigraph}
\usepackage{hyperref}
\usepackage{url}
\usepackage{booktabs}
\usepackage{makecell}    %
\usepackage[ruled,vlined]{algorithm2e} %

\usepackage[dvipsnames]{xcolor}
\usepackage{lineno}

\definecolor{darkblue}{rgb}{0, 0, 0.5}
\hypersetup{colorlinks=true, citecolor=darkblue, linkcolor=darkblue, urlcolor=darkblue}

\definecolor{lightblue}{RGB}{220,235,250}
\usepackage[most,skins,theorems]{tcolorbox}
\tcbset{
  takeawaysbox/.style={
    title=Takeaways,
    colback=lightblue!80,
    colframe=black,
    fonttitle=\bfseries\small,
    coltitle=white,
    colbacktitle=black,
    enhanced,
    attach boxed title to top left={xshift=2.5mm,yshift=-2.5mm},
    boxed title style={rounded corners, size=small, colframe=black, colback=black},
    width=\linewidth,
    arc=3.5mm
  }
}
\tcbset{
  taskbox/.style={
    title=Task Definition,
    colback=lightblue!80,
    colframe=black,
    fonttitle=\bfseries\small,
    coltitle=white,
    colbacktitle=black,
    enhanced,
    attach boxed title to top left={xshift=2.5mm,yshift=-2.5mm},
    boxed title style={rounded corners, size=small, colframe=black, colback=black},
    width=\linewidth,
    arc=2mm
  }
}

\usepackage{xspace}

\usepackage{cleveref}

\usepackage{amsfonts,amssymb,amsmath}			%
\usepackage{bm}

\usepackage{makecell}
\usepackage{graphics} %
\usepackage{xcolor}
\usepackage{bbold}
\usepackage{cite}
\usepackage{graphicx}
\usepackage{tikz}
\usepackage{amsthm}
\usepackage[skip=2pt,font=scriptsize,labelfont=footnotesize,textfont=footnotesize]{caption}
\usepackage[skip=2pt,font=scriptsize,labelfont=footnotesize,textfont=footnotesize]{subcaption}

\usepackage{bbm}
\usepackage{xspace}
\usepackage{wrapfig}
\usepackage{tabularx}
\usepackage{makecell} 
\usepackage{tablefootnote}
\usepackage{float}
\usepackage{stfloats}
\usepackage{tablefootnote}
\usepackage{float}
\usepackage{adjustbox}
\usepackage{enumitem}
\usepackage{booktabs} %
\usepackage{mathtools}
\usepackage{multirow}

\newcommand{\ifno}[1]{}

\usepackage{multirow}
\makeindex
\usepackage{makeidx}
\usepackage{acronym}

\usepackage{url}
\usepackage{graphicx}

\usepackage{microtype}

\usepackage{algorithm, algorithmic}%

\newcommand{\one}{\mathbf{1}}

\newcommand{\E}{\mathbb{E}}

\newcommand{\Dc}{\mathcal{D}}

\newcommand{\Lc}{\mathcal{L}}

\newcommand{\cL}{\mathcal{L}}

\newcommand{\be}{\begin{equation}}
\newcommand{\ee}{\end{equation}}
\newcommand{\ba}{\begin{array}}
\newcommand{\ea}{\end{array}}
\newcommand{\bad}{\begin{aligned}}
\newcommand{\ead}{\end{aligned}}

\usetikzlibrary{shapes,arrows}
\usetikzlibrary{arrows,calc}
\tikzset{
    block/.style = {draw, rectangle, 
        minimum height=1cm, 
        minimum width=2cm},
    input/.style = {coordinate,node distance=1cm},
    output/.style = {coordinate,node distance=3cm},
    arrow/.style={draw, -latex,node distance=2cm},
    pinstyle/.style = {pin edge={latex-, black,node distance=2cm}},
    sum/.style = {coordinate, node distance=1cm}
}

\newcommand{\method}{\textcolor{black}{Hybrid Post-Training}\xspace}
\newcommand{\framework}{\textcolor{black}{Unified Policy Gradient Estimator}\xspace}
\newcommand{\mot}{\textcolor{black}{HPT}\xspace}
\newcommand{\gpe}{\textcolor{black}{UPGE}\xspace}

\definecolor{darkred}{RGB}{200, 0, 0}

\NewDocumentCommand{\yuxin}{ mO{} }{\textcolor{red}{\textsuperscript{yuxin}\textbf{\small[#1]}}}
\NewDocumentCommand{\youbang}{ mO{} }{\textcolor{magenta}{\textsuperscript{youbang}\textbf{\small[#1]}}}
\NewDocumentCommand{\xingtai}{ mO{} }{\textcolor{RoyalPurple}{\textsuperscript{xingtai}\textbf{\small[#1]}}}

\usepackage{makecell}    %
\usepackage{natbib}
\usepackage{float}

\usepackage{fontawesome5}

\usepackage{colortbl}

\usepackage{amsmath,amssymb,mathtools}

\colorlet{moss}{green!40!black}

\SetCommentSty{AlgoCommentStyle}                 %
\SetKwComment{Hash}{\ttfamily\color{moss}\#\ }{} %
\usepackage{tablefootnote}

\definecolor{titlepurple}{HTML}{9673A6}

\title{{\fontsize{14pt}{15pt}\selectfont Towards a Unified View of Large Language Model Post-Training}}

\author{Xingtai Lv$^{1*}$, Yuxin Zuo$^{1*}$, Youbang Sun$^{1\dag}$, Hongyi Liu$^{1}$, Yuntian Wei$^{1}$, \\ 
\textbf{Zhekai Chen$^{1}$, Xuekai Zhu$^{1}$, Kaiyan Zhang$^{1}$, Bingning Wang$^{3}$}, \\
\textbf{Ning Ding$^{1,2\dag}$, Bowen Zhou$^{1,2\dag}$} \\
$^{1}$Tsinghua University, $^{2}$Shanghai AI Laboratory, $^{3}$WeChat AI\\
\\
\faGithub~\textbf{Code:}~\href{https://github.com/TsinghuaC3I/Unify-Post-Training}{TsinghuaC3I/Unify-Post-Training}
\\[0.5em]
\faEnvelope~\textbf{Mail:}~\text{lvxt24@mails.tsinghua.edu.cn}
}

\begin{document}

\ifcolmsubmission
\linenumbers
\fi

\maketitle

\renewcommand\thefootnote{}\footnote{$^{*}$ Equal Contributions. $^{\dag}$ Corresponding Authors.}

\begin{abstract}

Two major sources of training data exist for post-training modern language models: online~(model-generated rollouts) data, and offline~(human or other-model demonstrations) data. 
These two types of data are typically used by approaches like Reinforcement Learning (RL) and Supervised Fine-Tuning (SFT), respectively.
In this paper, we show that these approaches are not in contradiction, but are instances of a single optimization process.
We derive a \framework, and present the calculations of a wide spectrum of post-training approaches as the gradient of a common objective under different data distribution assumptions and various bias-variance tradeoffs. The gradient estimator is constructed with four interchangeable parts: stabilization mask, reference policy denominator, advantage estimate, and likelihood gradient. 
Motivated by our theoretical findings, we propose \method~(HPT), an algorithm that dynamically selects different training signals.
\mot is designed to yield both effective exploitation of demonstration and stable exploration without sacrificing learned reasoning patterns.
We provide extensive experiments and ablation studies to verify the effectiveness of our unified theoretical framework and \mot.
Across six mathematical reasoning benchmarks and two out-of-distribution suites, \mot consistently surpasses strong baselines across models of varying scales and families.

\end{abstract}

\begin{figure}[h]
    \centering
    \includegraphics[width=\linewidth]{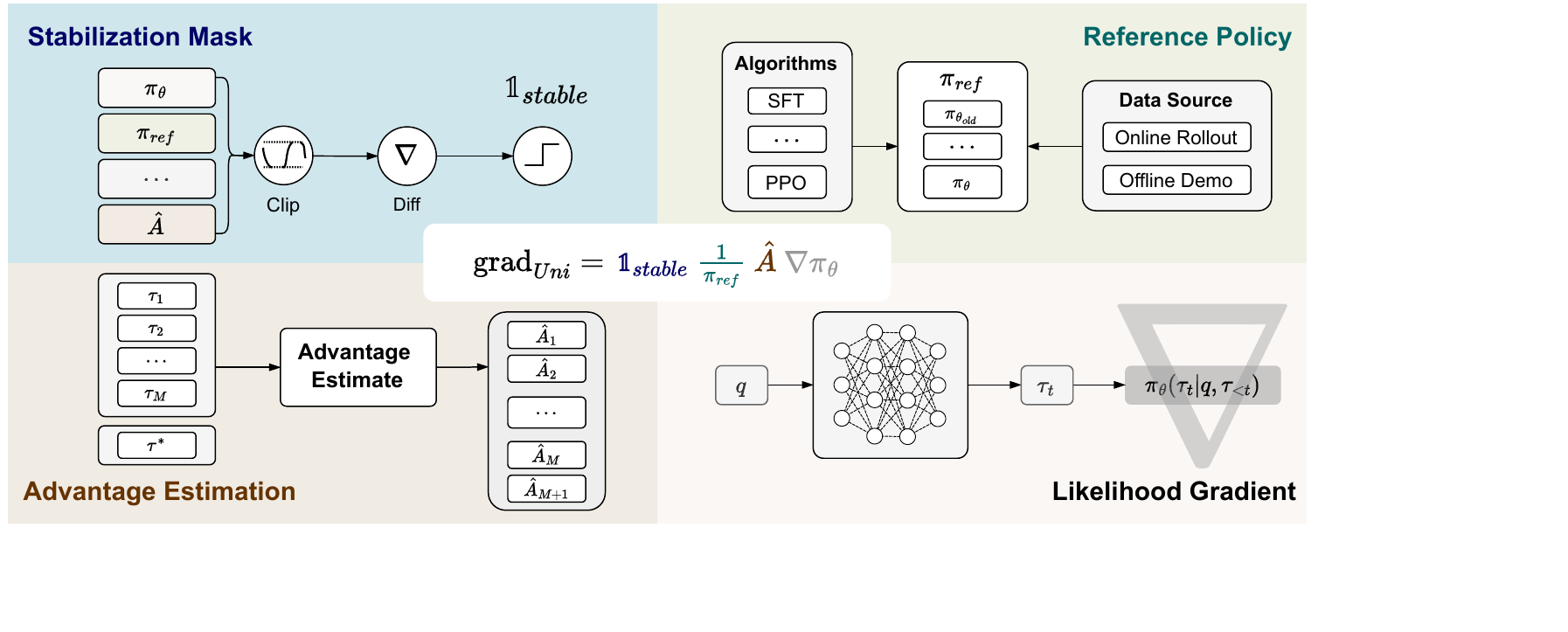}
    \caption{Illustration of the Unified Policy Gradient Estimator. The ``$\nabla$'' in the background of the Likelihood Gradient part refers to the calculation of the gradient with respect to the $\pi_\theta$. 
    }
    \label{fig:upge_overview}
\end{figure}

\vspace{-4mm}
\setlength{\epigraphwidth}{0.85\textwidth}
\noindent\hfill

\newpage
{
  \tableofcontents
}
\newpage

\section{Introduction}

Reinforcement Learning has played an integral role in enhancing the reasoning capabilities of large language models (LLMs) \citep{jaech2024openai, team2025kimi, guo2025deepseek}.
RL allows the model to freely explore the reasoning space in the post-training process and improve its performance based on the feedback provided in the environment. 
However, applying Reinforcement Learning directly to a base model (i.e., ``Zero RL'') \citep{zeng2025simplerlzooinvestigatingtamingzero} presupposes a certain level of inherent capability. This method often falters when applied to weaker models or tasks of high complexity, as the exploration process may fail to explore and discover meaningful reward signals. 
Conversely, the classical Supervised Fine-Tuning (SFT) \citep{wei2021finetuned} offers a direct and efficient method to distill knowledge from high-quality, human-annotated data, enabling models to rapidly and accurately fit the target distribution. Yet this approach often curtails the model's exploratory capabilities, potentially leading to overfitting on the demonstration data and compromising its generalization performance on out-of-distribution inputs. 
Consequently, a sequential ``SFT-then-RL'' pipeline \citep{yoshihara2025practical} has emerged as the standard, adopted by numerous state-of-the-art open-source models. While effective, this multi-stage process, which first elevates the model's capabilities through SFT before refining them with RL, is notoriously resource-intensive and usually requires careful tuning to ensure effectiveness.

To circumvent these challenges, recent works have focused on integrating SFT or SFT-style imitation learning losses directly with RL objectives \citep{yan2025learning, fu2025srft, zhang2025policy}.
In these approaches, the model is updated using a composite loss function. The balance between the imitation and exploration components is governed by various strategies, including a fixed coefficient, a predefined schedule, a dynamic adjustment based on entropy, or a learnable parameter. 
These works predominantly treat the SFT and RL losses as two distinct objectives. And a detailed analysis of \textit{why these two learning signals can be effectively combined within a unified optimization process} remains largely unexplored.

Despite their distinct mathematical formulations, we find that the gradient calculations from these approaches can be viewed as a single, unified form. Inspired by Generalized Advantage Estimator \citep{schulman2015high}, we introduce \framework (UPGE), a framework that formally subsumes the gradients of various post-training objectives into one generalized expression. We provide analysis to show that the various forms of gradients are, in fact, not conflicting.
Instead, they act as complementary learning signals that can jointly guide the optimization process.
However, these gradient estimators possess different characteristics, and there exists a bias-variance tradeoff in their respective gradient components. 
Building upon this unified perspective, we propose \method (\mot), a hybrid algorithm to dynamically choose more desirable training signals by adapting a mixing ratio between the SFT and RL losses.
This mechanism allows \mot to be intrinsically adaptive to models of varying capabilities and data of differing complexities.

We implement a simple instance of \mot, which adaptively switches between SFT and RL based on rollout accuracy, and empirically demonstrate that it achieves strong results.
Our empirical evaluations demonstrate that \mot surpasses strong baselines such as SFT$\rightarrow$GRPO and LUFFY with Qwen2.5-Math-7B, achieving a 7-point gain over our strongest baseline on AIME 2024.
Moreover, \mot also yields substantial improvements even on relatively smaller and weaker models, including Qwen2.5-Math-1.5B and Llama3.1-8B.
Through detailed training dynamics and illustrative training visualizations, we clearly reveal the features and underlying mechanisms of \mot.
The following are several key takeaways:

\begin{tcolorbox}[takeawaysbox]
\begin{enumerate}[leftmargin=1em]
    \item \gpe provides a theoretical unification of a wide spectrum of post-training algorithms, covering both SFT and RL losses within a single formulation~\textbf{($\S$~\ref{sec:unified-view})}.
    \item \mot is capable of outperforming previous post-training and mixed-policy algorithms across a diverse range of models~\textbf{($\S$~\ref{sec:experiments})}.
    \item Dynamic integration of SFT and RL in \mot achieves the highest \emph{Pass@1024}, facilitating enhanced exploration and generalization of the model~\textbf{($\S$~\ref{sec:explor_exploit})}.
\end{enumerate}

\end{tcolorbox}

\section{A Unified View on Post-Training Algorithms} \label{sec:unified-view}

In this section, we adopt a unified perspective to understand both Supervised Fine-Tuning (SFT) and Reinforcement Learning (RL) as post-training objectives.
We present the gradient calculations of various post-training approaches in Table \ref{tab:grad_compare}, with exact derivations of classical approaches presented in the Appendix \ref{sec:appdendix_grad_cal}. From the table, it can be seen that policy gradient calculations for LLM post-training can be written in a unified policy gradient form. 

\begin{tcolorbox}[takeawaysbox]
We propose a unified framework for the gradient calculation of LLM post-training, named the \textcolor{darkred}{Unified Policy Gradient Estimator}:
$$ \text{grad}_{Uni} = \mathbb{1}_{stable} \frac{1}{\pi_{ref}} \hat{A} \nabla \pi_{\theta}.$$All gradient calculations can be written in the unified form.
\end{tcolorbox}

In the following sections, we further show that the differences between different gradient calculations can be broken down into four distinct components.
We theoretically derive the Unified Policy Gradient Estimator from a common objective and provide a detailed analysis of its gradient components. Based on this unified perspective, we then propose the \method (\mot) algorithm.

\begin{table*}[!t]
    \centering
    \caption{Theoretical unified view of various post-training algorithms.}
    \label{tab:grad_compare}
    \resizebox{\textwidth}{!}{
    \begin{tabular}{cccc}
    \toprule
         \textbf{Algorithm} & \textbf{Reference Policy} &\textbf{Advantage Estimate} & \textbf{Unified Policy Gradient Estimator} \\
        \midrule
          SFT & $\pi_{ref} = \pi_\theta$& $\hat{A}_{SFT} \equiv 1$ & $\nabla \mathcal{J}_{SFT}(\theta) = \nabla \pi_\theta (\tau ) \frac{\hat{A}_{SFT} = 1}{\pi_\theta (\tau )}$\\
        \midrule
        \multicolumn{4}{c}{\textbf{Online Reinforcement Learning Methods}} \\
        \midrule
          \makecell{PPO\\\citep{schulman2017proximal}} & $\pi_{ref} = \pi_{\theta_{old}}$&
        $\hat{A}_{PPO} = \text{GAE \citep{schulman2015high}}$ & 
        $\nabla \mathcal{J}_{PPO} = \nabla \pi_\theta (\tau) \frac{\hat{A}_{PPO} \mathbb{1}_{\text{Clip}}}{\pi_{ref} (\tau)}$\\
        \midrule
          \makecell{GRPO\\\citep{shao2024deepseekmath}} & $\pi_{ref} = \pi_{\theta_{old}}$&
        $\hat{A}_{GRPO} = \frac{R(\tau_{j}) - \text{mean} (\{R(\tau_{j})\}_{G_{on}})}{\text{std} (\{R(\tau_{j})\}_{G_{on}})}$ & 
        $\nabla \mathcal{J}_{GRPO} = \nabla \pi_\theta (\tau) \frac{\hat{A}_{GRPO} \mathbb{1}_{\text{Clip}}}{\pi_{ref} (\tau)}$\\
        \midrule
         \makecell{REINFORCE\\\citep{ahmadian2024back}} & $\pi_{ref} = \pi_{\theta}$&
        $\hat{A}_{REINFORCE} = \pm 1$ & 
        $\nabla \mathcal{J}_{REF.}(\theta) = \nabla \pi_\theta (\tau) \frac{\hat{A}_{REF.}}{\pi_\theta (\tau)}$\\
        \midrule
         \makecell{CISPO\\\citep{chen2025minimax}} & $\pi_{ref} = \pi_{\theta_{old}}$&
        $\hat{A}_{CISPO} = \hat{A}_{GRPO}$ & 
        $\nabla \mathcal{J}_{CISPO} = \nabla \pi_\theta (\tau) \frac{\hat{A}_{CISPO} \mathbb{1}_{\text{CIS-Mask}}}{\pi_{ref} (\tau)}$\\
        \midrule
         \makecell{GSPO\\\citep{zheng2025group}} & $\pi_{ref} = \pi_\theta \left(\frac{\pi_{\theta_{old}} (\tau_{i,j} |q_i ) }{\pi_{\theta} (\tau_{i,j} |q_i )}\right)^{1/|\tau_{i,j}|}  $&
        $\hat{A}_{GSPO} = \hat{A}_{GRPO}$ & 
        $\nabla \mathcal{J}_{GSPO} = \nabla \pi_\theta (\tau) \frac{\hat{A}_{GSPO} \mathbb{1}_{\text{Seq-Clip}}}{\pi_{ref} (\tau)}$\\
        \midrule
        \multicolumn{4}{c}{\textbf{Offline/Online Reinforcement Learning Methods}} \\
        \midrule
         \makecell{SRFT (Offline)\\\citep{fu2025srft}} & $\pi_{ref} \equiv 1$&
        $\hat{A}_{SRFT} = \frac{R(\tau_{j}) - \text{mean} (\{R(\tau_{j})\}_{G_{on}\cup G_{off}})}{\text{std} (\{R(\tau_{j})\}_{G_{on}\cup G_{off}})}$ & 
        $\nabla \mathcal{J}_{SRFT} = \nabla \pi_\theta (\tau) \frac{\hat{A}_{SRFT}}{\pi_{ref} (\tau) = 1}$\\
        \midrule
         \makecell{LUFFY (Offline)\\\citep{yan2025learning}} & $\pi_{ref} \equiv 1$&
        $\hat{A}_{LUFFY} = \hat{A}_{SRFT}$ & 
        $\nabla \mathcal{J}_{LUFFY} = \nabla \pi_\theta (\tau) \frac{\hat{A}_{LUFFY}}{\pi_{ref} (\tau) = 1} f_{\text{shape}}'$ \\
    \bottomrule
    \end{tabular}
    }
\end{table*}

\subsection{Components of the Unified Policy Gradient Estimator}
\label{Components of the Unified Policy Gradient Estimator}

We present the Unified Policy Gradient Estimator, our unified framework for gradient calculations. In Table \ref{tab:grad_compare}, we list a series of fundamental and well-studied post-training methods, divided into SFT and two types of RL processes. Apart from providing the closed-form policy gradients of these methods, we also present the decomposition of these methods with detailed components. It can be seen that these seemingly different methods in fact share common components and that all gradients follow our proposed unified framework.

In this paper, we divide the unified gradient into four terms: \textit{stabilization mask}, \textit{reference policy}, \textit{advantage estimate}, and \textit{likelihood gradient}. We address each of the terms below.

\paragraph{Stabilization Mask $\mathbb{1}_{stable}$} Starting from PPO \citep{schulman2017proximal}, the stabilization mask was first derived as an approximation of the TRPO Algorithm \citep{schulman2015trust}.
In practice, the PPO clipping addresses the instability issue during RL training by turning off the current gradient when the current iterate is considered unsafe. In consequent works in Table \ref{tab:grad_compare}, many have provided their modifications on the stability mask, usually motivated by empirical evaluations.

\paragraph{Reference Policy Denominator $\pi_{ref}$} The second term in our unified estimator is the reference policy on the denominator. We note that our notion of reference policy differs from the commonly used rollout policy $\pi_{\theta_{old}}$, for which we provide a discussion in Section \ref{sec:component-analysis}. This denominator denotes a token-level reweight coefficient, usually in the form of an inverse probability. There are multiple choices for this coefficient. For the case of SFT, the policy denominator uses the current policy $\pi_\theta (\tau )$. This is a result of $\Lc = -\log (\pi_\theta (\tau ))$ as the objective function. For the case of PPO-style online RL algorithms, generally, the policy denominator uses the rollout policy $\pi_{\theta_{old}} (\tau )$. Due to the unavailability of $\pi_{ref} (\tau )$ in the offline demonstration dataset, most offline RL algorithms simply assume $\pi_{ref} (\tau )= 1$ for the denominator.

\paragraph{Advantage Estimate $\hat{A}$} 
In traditional RL, the advantage evaluates the additional benefit of taking the current action given the current state. For the context of LLMs, most of the advantage estimation is sequence-level rather than token-level, and measures the quality of the current response sequence. Similar to traditional RL literature, the post-training process seeks to maximize the likelihood of generating positive sequences with high advantage and minimize negative sequences.

\paragraph{Likelihood Gradient $\nabla \pi_\theta (\tau)$} The policy gradient term is a general term which maps gradient information from the actions to the model parameters $\theta$. It is crucial for back-propagating the objective signals to the network weights, and is kept the same across all gradient calculations.

\subsection{Derivation of the Unified Policy Gradient Estimator}
\label{sec:shared-common-objective}

We begin from a simple and common objective shared by all post-training algorithms: improve the likelihood of positive trajectories and decrease the likelihood of negative trajectories such that the total reward in expectation $\max_{\theta} \mathcal{J} (\theta) := \mathbb{E} [r(\tau|q)]$ is maximized.
From this starting point, we theoretically derive our Unified Policy Gradient Estimator. We then show that SFT and RL objectives are not in conflict, and they can be optimized jointly within a single loss.

\paragraph{Common Objective.}
We model the post-training as a process to maximize the expected success rate while keeping the model policy closely adhering to a demonstration dataset (behavior policy) $\pi_\beta$:
\begin{equation}
\label{eq:master_obj}
\begin{aligned}
\mathcal{J}_{\mu}(\theta)
&= \E_{\tau\sim \pi_\theta(\cdot\mid q)}\!\big[r(\tau\mid q)\big]
\;-\; \mu\,\mathrm{KL}\!\big(\pi_\beta(\cdot\mid q)\,\|\,\pi_\theta(\cdot\mid q)\big),
\qquad \mu\ge 0,
\end{aligned}
\end{equation}
where $q\!\sim\!\mathcal{D}$ denotes the question from a given distribution, $\tau$ denotes a trajectory, $r$ denotes the (binary/real) score, and $\pi_\beta$ denotes behavior policy from demonstration.

\paragraph{Gradient of the Common Objective.}
Differentiating and rearranging Equation~\ref{eq:master_obj} (full derivation in Appendix~\ref{app:derivation}), we obtain
\begin{equation}
\label{eq:master_grad_pi_measure}
\begin{aligned}
\nabla_\theta \mathcal{J}_{\mu}(\theta)
&= \E_{\tau\sim \pi_\theta}\!\Big[r(\tau\mid q)\,\nabla_\theta \log \pi_\theta(\tau\mid q)\Big]
\;+\; \mu\,\E_{\tau\sim \pi_\beta}\!\big[\nabla_\theta \log \pi_\theta(\tau\mid q)\big].
\end{aligned}
\end{equation}

\paragraph{From gradient to the Unified Policy Gradient Estimator.}
Applying the measure-change identity (detailed in Appendix~\ref{app:derivation}) with the reference policy $\pi_{ref}$ which we mentioned in Section~\ref{Components of the Unified Policy Gradient Estimator} and using $\nabla \log \pi_\theta=(1/\pi_\theta)\nabla \pi_\theta$ yields the gradient:
\begin{equation}
\label{eq:upge_form}
\nabla_\theta \mathcal{J}_{\mu}(\theta)
=
\E_{\tau\sim \pi_{ref}(\cdot\mid q)}
\!\left[
\frac{1}{\pi_{ref}(\tau\mid q)}\,
\widehat{A}_{uni}(\tau,q)\,
\nabla_\theta \pi_\theta(\tau\mid q)
\right],
\end{equation}
 with the unified advantage
\begin{equation}
\label{eq:A_uni}
\widehat{A}_{uni}(\tau,q)
=
\underbrace{r(\tau\mid q)}_{\widehat{A}_{\mathrm{RL}}(\tau,q)}
\;+\;
\underbrace{\mu\,\frac{\pi_\beta(\tau\mid q)}{\pi_\theta(\tau\mid q)}}_{\widehat{A}_{\mathrm{SFT}}(\tau,q)}.
\end{equation}
In many RL works, the raw score $r(\tau\mid q)$ is replaced by a more structured advantage to reduce variance, provide relative credit assignment within a rollout group, and stabilize step sizes. For example, GRPO uses group-wise normalization:
\begin{equation}
\widehat{A}_{\mathrm{GRPO}}(\tau_j,q) = \frac{R(\tau_j)-\mathrm{mean}(\{R(\tau)\}_{G_{\mathrm{on}}})}{\mathrm{std}(\{R(\tau)\}_{G_{\mathrm{on}}})}.
\end{equation}

When trust-region stabilization masks, as induced by PPO clipping, are inserted multiplicatively without altering the target objective, we obtain our Unified Policy Gradient Estimator:
\begin{equation}
\label{eq:upge_masked}
\begin{aligned}
\mathrm{grad}_{uni}
&=
\E_{\tau\sim \pi_{ref}(\cdot\mid q)}\!\left[
\mathbb{1}_{stable}(\tau,q)\,
\frac{1}{\pi_{ref}(\tau\mid q)}\,
\widehat{A}_{uni}(\tau,q)\,
\nabla_\theta \pi_\theta(\tau\mid q)
\right]
\\[2pt]
&=\;
\mathbb{1}_{stable}\;
\frac{1}{\pi_{ref}}\;
\hat{A}\;
\nabla \pi_\theta.
\end{aligned}
\end{equation}
The trust-region surrogate that produces the mask is given in Appendix~\ref{app:ppo-mask-subsec}.

The gradient in \eqref{eq:master_grad_pi_measure} is the sum of two terms: (i) a reward + trust-region term sampled from $\pi_\theta$ and (ii) a data-adherence (SFT) term sampled from $\pi_\beta$. Both terms map to the same estimator via \eqref{eq:upge_form}–\eqref{eq:upge_masked} by choosing $\pi_{ref}$ accordingly (e.g., $\pi_{\theta_{old}}$ for on-policy trust-region updates and $\pi_\beta$ for SFT/offline updates). Therefore, SFT and RL optimize a single Common Objective \eqref{eq:master_obj} and can be trained jointly within one loss without intrinsic conflict.

\subsection{Gradient Component Analysis} \label{sec:component-analysis}

\begin{tcolorbox}[takeawaysbox]
\begin{enumerate}[leftmargin=1em]
    \item While all algorithms share the same Common Objective, bias-variance trade-offs still exist across current instances for different components of the unified gradient estimator.
    \item We can improve the post-training process by constructing a better and more suitable estimation of the policy gradient.
\end{enumerate}

\end{tcolorbox}

Across the wide spectrum of algorithms contained in our previous discussions and Table \ref{tab:grad_compare}, it can be inferred that the four components that construct the unified gradient estimator are motivated by different procedures in the post-training process. 
To better illustrate the relationship between the different processes with the respective components of our unified gradient, we present Figure \ref{fig:upge_overview}.

We divide the post-training process of LLMs into the four steps shown in Figure \ref{fig:upge_overview}: i) First, the LLM makes the decision on its data source, either to use data from an offline demonstration dataset, from self-generated rollout data, or a mixture of both. In this process, the policy likelihood $\pi_\theta$ of the data with respect to the current LLM is generated. ii) Given the data source used for data generation, a reference policy $\pi_{ref}$ is calculated. iii) After data collection is complete, the algorithm calculates the advantage estimation $\hat{A}$ for each token/sequence. iv) Lastly, the algorithm may choose to apply an additional masking procedure $\mathbb{1}_{stable}$ to disable the gradient calculation of various tokens, which could lead to theoretical or numerical stability issues. After these four steps, the components are collected to construct the policy gradient $\text{grad}_{Uni}$, which is used to update the LLM in the system. 
Similar to GAE presented in \citep{schulman2015high}, multiple instantiations exist to estimate the policy gradient. However, different component selections introduce various degrees of bias and variance, where a trade-off is often encountered. 
We provide the following discussion on key components of the unified gradient below.

\paragraph{Reference Policy Calculation}
Practically speaking, the reference policy denominator places a weight on each token-level update such that any token with a smaller probability, often implying more significance, is weighted more. 
SFT and REINFORCE assign weights inversely proportional to the current policy $\pi_\theta$, enforcing a bigger update when the model outputs a small probability. 
On the other hand, when the data is generated with an outdated model, algorithms such as PPO assign weights inversely proportional to the rollout policy $\pi_{\theta_{old}}$, and offline RL does not assign additional weights for tokens.

Theoretically, the reference policy is usually set given the source of the dataset and/or the rollout policy. 
For online RL methods that train purely with on-policy data, such as REINFORCE \citep{ahmadian2024back}, uses $\frac{1}{\pi_\theta}$, which produces an unbiased estimate for gradient calculation. However, these methods usually suffer from high variance.
For PPO-style online RL algorithms, the reference policy refers to the rollout policy, which is a result of importance sampling.
PPO is a numerically simplified version of TRPO \citep{schulman2015trust}. PPO makes conservative updates that effectively reduce variance. However, the important sampling ratio is in fact theoretically ill-posed and could introduce systematic bias, as discussed in GSPO \citep{zheng2025group}. GSPO has also proposed a novel calculation for $\pi_{ref}$, as shown in Table \ref{tab:grad_compare}.
On the other hand, in the offline setting, the choice for reference policy $\pi_{ref}$ is limited, since the algorithm generally has no access to the rollout policy. If we are given the assumption that the offline data evenly covers the entire state-action rollout space, then the importance sampling ratio $r(\theta) = \frac{\pi_\theta (\tau ) }{\pi_{ref} (\tau )}$ reduces to $\pi_\theta (\tau )$ by setting constant $\pi_{ref} (\tau ) = 1$. Notably, it is apparent that setting $\pi_{ref} (\tau ) = 1$ introduces much bias at the cost of numerical stability.
For the SFT case, we can consider that the domain-specific dataset is generated with respect to the expert policy $\pi^\ast$; therefore, no weighted sampling is required. 
Neither of the two approaches is entirely theoretically justified, from an RL perspective; both require a lower bound on the state-action visitation of all the possible state-action pairs \citep{kakade2003sample}, which can not be satisfied due to the severely limited datasets in practice.

Apart from the strong connection to data source and sampling polices, some studies employ a hand-crafted reweight factor within the reference policy denominator. These works \citep{yan2025learning, zhang2025policy} typically find desirable token properties and purposefully place a higher/lower weight on these desirable/undesirable tokens, respectively.

\paragraph{Choice of Stabilization Mask}
The clipping operation introduced in PPO was the first to explicitly add a stop gradient operation on LLM post-training. Clipping gradient estimation where the importance sampling strays too far from $1$ is an effective approach to address high variances. However, this aggressive clipping behavior has been criticized by some to be overly conservative: Both DAPO \citep{yu2025dapo} and CISPO \citep{chen2025minimax} stated that the classical PPO approach drops all the tokens corresponding to large model updates, and that many such tokens are in fact crucial for stabilizing entropy and facilitating scalable RL. DAPO presented a slight modification to the clipping threshold, and CISPO further extended the notion of token-wise mask, where more granular tuning was introduced to decide whether gradients from specific tokens should be dropped. The recent work of \citet{cui2025entropy} has demonstrated that many existing algorithms negatively impact the output entropy during training and introduced Clip-Cov, adding another clipping mechanism to address the entropy-collapse encountered in training. While these methods demonstrated performance enhancements in practice, they also provide additional sources of bias.

On the other hand, works such as GSPO \citep{zheng2025group} have stated that the PPO-style clipping is inherently noisy and inefficient for sample exploitation: GSPO clips a much larger fraction of tokens and yet demonstrated superior training efficiency.

In addition, post-training algorithms using offline data have chosen to purposefully remove the clipping from training, mostly guided by performance. Though setting $\pi_{ref} (\tau ) = 1$ as the policy denominator does effectively reduce the instability in gradient calculations.

\paragraph{Advantage Estimation}
There are two commonly used settings for estimating the sequence-level advantage function: the fixed advantage setting and the adaptive advantage setting. The fixed setting considers $\hat{A} = \pm 1$ given the rule-based verification, which is adapted by REINFORCE and implicitly by SFT (where all sequences are positive samples). Alternatively, recent studies have focused on using adaptive advantage estimations, performing re-centering or normalization based on the performance of the current rollout group. Notably, GRPO and its variants, such as DAPO \citep{yu2025dapo} and LUFFY \citep{yan2025learning}, use unit normalization such that the advantage estimation of the group has a unit standard deviation. Other approaches, such as Dr. GRPO \citep{liu2025understanding}, RLOO \citep{ahmadian2024back}, and REINFORCE++ \citep{hu2025reinforce++}, claim that dividing the standard deviation introduces a difficulty bias and that only recentering is adequate. 

Apart from sequence-level advantage estimate $\hat{A}_{i,j}$, recent works \citep{wang2025beyond, yang2025treerpo, sun2025ktae} have also adapted a more granular token-level advantage estimate $\hat{A}_{i,j,t}$ to a varying degree of success.

\paragraph{A Combination of Gradient Estimators}

Although bias-variance trade-offs exist for the gradient estimator, we state that, given data distribution assumptions and sufficient data samples, all policy gradient estimators covered in our framework should result in an effective direction of improvement for the Common Objective. 
To effectively reduce the variance and bias for each policy update, we can treat instances of policy gradient as different noisy measurements of the true policy gradient, and perform a weighted average to generate a more accurate gradient estimation, similar to complementary filters \citep{marantos2015uav}.

However, the complexity of LLM RLVR introduces additional challenges. The current state of the behavior policy $\pi_\theta$ and its relationship with the respective tasks also greatly impacts the bias-variance tradeoff of each instance of the gradient estimator.
For instance, RL-zero is significantly more effective for the Qwen model series compared to LlaMA, but SFT is effective for both methods \citep{zeng2025simplerlzooinvestigatingtamingzero}; SFT $\rightarrow$ RL and RL $\rightarrow$ SFT also yield significantly different results on the same LLM \citep{fu2025srft}.
We argue that for constructing a post-training algorithm with better effectiveness and efficiency, a dynamic and adaptive mechanism is crucial to construct optimal gradient components.

\subsection{Hybrid Post-Training with Performance Feedback}
\label{sec:method}

Our unified perspective above shows that different post-training losses have the same optimization objective with different characteristics.
Inspired by this view, we propose the \method (\mot) algorithm.
We use a mixed loss $\mathcal{L} = \alpha \mathcal{L}_{\mathrm{RL}} + \beta \mathcal{L}_{\mathrm{SFT}}$, which contains the weighted on-policy RL loss $\mathcal{L}_{\mathrm{RL}}$ and SFT loss $\mathcal{L}_{\mathrm{SFT}}$, to optimize the target LLM $\pi_\theta$. The weights of the two losses ($\alpha$ and $\beta$) are determined by the real-time sampling performance of the model.

\paragraph{Performance on Single Question.} 
For any question $q$ provided to the LLM, we first obtain both a supervising trajectory $\tau^\star$ and the model's performance $P$ on the question.
Specifically, we draw $n$ on-policy trajectories $\{\tau_i\}^n_{i=1} \sim \pi_{\theta}(\cdot\mid q)$ and evaluate them with a verifier $v:\tau_i\to\{0,1\}$. This verifier is the same as the rule-based reward function and the model's performance $P$ is defined as the mean of these $n$ verification scores:
\begin{equation}
\label{eq: compute r}
v(\tau_i) = R(\tau_i) =\begin{cases}
 1 & \text{if } \tau_i \text{ contains the correct answer of } q  \\
 0 & \text{otherwise}
\end{cases} 
\end{equation}
\begin{equation}
P = \frac{1}{n}\sum_{i=1}^n v(\tau_i)
\end{equation}
Intuitively, $P$ indicates how well the current policy performs on $q$ across multiple trajectories.

\paragraph{Feedback Coefficients.} Then, we obtain the coefficients of on-policy RL loss $\alpha$ and SFT loss $\beta$ based on the performance feedback:
\begin{equation}
\alpha = f(P), \quad \beta = g(P),
\end{equation}
where the $f$ and $g$ are the specific feedback functions.
Experientially, when the model demonstrates strong capability, it is advantageous to emphasize on-policy RL to foster exploration; conversely, when the model’s competence is limited, SFT should take precedence to ensure correct guidance. Consequently, $f$ ought to be positively correlated with $P$, whereas $g$ should exhibit a negative correlation. In this paper, we employ a pair of simple yet empirically effective switch functions $f$ and $g$:
\begin{equation}
\alpha = f(P) =\begin{cases}
 1 & \text{if } P>\gamma   \\
 0 & \text{if } P\leq\gamma 
\end{cases} , \quad
\beta = g(P) =\begin{cases}
 1 & \text{if } P\leq\gamma    \\
 0 & \text{if } P>\gamma  
\end{cases}
\end{equation}
The switch gate $\gamma$ enables the model to perform SFT when its performance falls below a predefined threshold, and RL otherwise.

\paragraph{Mixed Loss.} Finally, we calculate the RL loss $\mathcal{L}_{\mathrm{RL}}$ with the already generated $n$ on-policy trajectories $\tau_i$ and SFT loss $\mathcal{L}_{\mathrm{SFT}}$ with the supervising trajectory $\tau^\star$, and we use Dr. GRPO as the on-policy RL algorithm:
\begin{equation}
\label{eq: rl loss}
\mathcal{L}_{\mathrm{RL}}
=
-\frac{1}{n}\sum_{i=1}^{n}\sum_{t=1}^{|\tau_i|}
\min\!\Big(
r_{i,t}\,A_{i,t},\;
\operatorname{clip}\!\big(r_{i,t},\,1-\epsilon,\,1+\epsilon\big)\,A_{i,t}
\Big)
\end{equation}
\begin{equation}
\label{eq: sft loss}
\mathcal{L}_{\mathrm{SFT}}=
-\frac{1}{\lvert \tau^\star \rvert}
\sum_{t=1}^{\lvert \tau^\star \rvert}
\log \pi_\theta\!\left(\tau^\star_t \,\middle|\, q, \tau^\star_{<t}\right)
\end{equation}
where $r_{i,t}=\frac{\pi_{\theta}\left(\tau_{i,t}\mid q, \tau_{i,<t}\right)}{\pi_{\theta_{old}}\left(\tau_{i,t}\mid q, \tau_{i,<t}\right)}$ is the per-token importance sampling ratio, $A_{i,t}\equiv A_i =\frac{  R(\tau_i)-\mathrm{mean}\left(\left\{\, R(\tau_i)\ \middle|\ i=1,2,\ldots,n \right\}\right)}{\mathrm{std}\left(\left\{\, R(\tau_i)\ \middle|\  i=1,2,\ldots,n \right\}\right)}$ is the advantage  and $\epsilon$ is the clip gate hyperparameter.
The mixed loss is then obtained by taking a weighted average of these two losses using performance feedback coefficients $\alpha$ and $\beta$: 
\begin{equation}
\mathcal{L} = \alpha \mathcal{L}_{\mathrm{RL}} + \beta \mathcal{L}_{\mathrm{SFT}}
\end{equation}

\begin{algorithm}[t]
\caption{The \method\ (\mot) Algorithm}
\label{alg:upt}
\SetAlgoLined
\DontPrintSemicolon
\KwIn{Pretrained LLM (policy) $\pi_\theta$; SFT dataset $\mathcal{D}_{\mathrm{SFT}}=\{(q,\tau^\star)\}$ with supervising trajectories $\tau^\star$; verifier $v$; on\mbox{-}policy samples number $n$; total training steps $T$; feedback functions $f$ and $g$; learning rate $\eta$}
\KwOut{Fine\mbox{-}tuned policy $\pi_{\theta^\ast}$.}

\For{$t = 1$ \KwTo $T$}{
  \For{$i=1$ \KwTo $n$}{
    Sample trajectory $\tau_i \sim \pi_\theta(\cdot \mid q)$\;
    Evaluate with verifier (rule-based reward): \;
    \Indp
      $v(\tau_i) \leftarrow R(\tau_i) \in \{0,1\}$\;
    \Indm
  }
  $P \leftarrow \frac{1}{n}\sum_{i=1}^{n} v(\tau_i)$\;
  $\alpha \leftarrow f(P), \quad \beta \leftarrow g(P)$\; \Hash{Performance feedback on question $q$}
  
  Compute on-policy RL loss $\mathcal{L}_{\mathrm{RL}}$ using rollouts $\{\tau_i\}$ and normalized advantages derived from $\{R(\tau_i)\}$.

  Compute SFT loss $\mathcal{L}_{\mathrm{SFT}}$ on the supervising trajectory $\tau^\star$.
  
  $\mathcal{L} \leftarrow \alpha\,\mathcal{L}_{\mathrm{RL}} + \beta\,\mathcal{L}_{\mathrm{SFT}}$\; \Hash{Mixed loss with performance feedback coefficients}

  $\theta \leftarrow \theta - \eta\,\nabla_{\theta}\mathcal{L}$\
}
\Return{$\pi_{\theta^\ast}$}\;
\end{algorithm}

\section{Experiments}
\label{sec:experiments}

\subsection{Experimental Setup}

\paragraph{Models}
To evaluate the generalizability of \mot across different backbone models, we conduct experiments using Qwen and LLaMA models of various scales.
The models we experiment with are as follows:
\begin{itemize}[leftmargin=2em]
    \item \textbf{Qwen Family:} Qwen2.5-Math-1.5B, Qwen2.5-Math-7B~\citep{yang2024qwen2};
    \item \textbf{LLaMA Family:} LLaMA-3.1-8B~\citep{grattafiori2024llama};
\end{itemize}

\begin{table*}[!t]
\definecolor{bluishyellow}{rgb}{0.84, 0.92, 0.85}
\definecolor{bluishcyan}{rgb}{0.74, 0.93, 0.95}
\centering
\caption{In-distribution and out-of-distribution performance of \mot and baselines on Qwen2.5-Math-7B. $^*$ means the results are taken from the corresponding paper.}
\label{tab:Qwen7B-results}
\resizebox{\textwidth}{!}{
\begin{tabular}{lccccccc|ccc}
\toprule
\multirow{2}{*}{\textbf{Model}} & \multicolumn{7}{c}{\textbf{In-Distribution}} & \multicolumn{3}{c}{\textbf{Out-of-Distribution}} \\
\cmidrule(lr){2-8} \cmidrule(lr){9-11}
 & \textbf{AIME 24} & \textbf{AIME 25} & \textbf{AMC} & \textbf{MATH-500} & \textbf{Minerva} & \textbf{Olympiad} & \textbf{Avg} & \textbf{ARC-c} & \textbf{GPQA} & \textbf{Avg} \\
\midrule
Qwen2.5-Math-7B
  & $12.3$ & $4.7$  & $33.0$ & $43.6$ & $8.8$  & $13.6$ & $19.3$ & $30.9$ & $28.3$ & $29.6$ \\
\midrule
\multicolumn{1}{r}{SFT}                                     
  & $25.1$ & $\mathbf{22.8}$ & $56.1$ & $84.2$ & $33.8$ & $44.7$ & $44.5$ & $67.4$ & $25.3$ & $46.4$ \\
\multicolumn{1}{r}{GRPO}                                  
  & $19.4$ & $13.8$ & $59.1$ & $81.8$ & $38.2$ & $46.2$ & $43.1$ & $81.2$ & $36.4$ & $58.8$ \\
\multicolumn{1}{r}{SFT $\rightarrow$ GRPO}
  & $25.7$ & $21.6$ & $62.2$ & $84.6$ & $38.2$ & $46.8$ & $46.5$ & $67.7$ & $30.8$ & $49.3$ \\
\midrule
\multicolumn{1}{r}{LUFFY} & $26.1$ & $21.8$ & $66.2$ & $88.4$ & $41.9$ & $54.1$ & $49.8$ & $80.8$ & $39.4$ & $60.1$ \\
\multicolumn{1}{r}{SRFT} & $18.4$ & $15.5$ & $55.9$ & $83.8$ & $42.6$ & $48.9$ & $44.2$ & $80.5$ & $36.8$ & $58.7$ \\
\rowcolor{lightblue!100}\multicolumn{1}{r}{\mot} & $\mathbf{33.0}$ & $21.9$ & $\mathbf{69.4}$ & $\mathbf{89.2}$ & $\mathbf{46.0}$ & $\mathbf{56.9}$ & $\mathbf{52.7}$ & $\mathbf{81.6}$ & $\mathbf{42.9}$ & $\mathbf{62.3}$ \\
\midrule
\textcolor{gray}{Qwen2.5-Math-7B-Ins.}
  & \textcolor{gray}{$11.8$} & \textcolor{gray}{$9.8$} & \textcolor{gray}{$48.3$} & \textcolor{gray}{$83.2$} & \textcolor{gray}{$34.2$} & \textcolor{gray}{$39.3$} & \textcolor{gray}{$37.8$} & \textcolor{gray}{$72.7$} & \textcolor{gray}{$29.3$} & \textcolor{gray}{$51.0$} \\
\textcolor{gray}{PRIME-Zero$^*$}                               
  & \textcolor{gray}{$17.0$} & \textcolor{gray}{$12.8$} & \textcolor{gray}{$54.0$} & \textcolor{gray}{$81.4$} & \textcolor{gray}{$39.0$} & \textcolor{gray}{$40.3$} & \textcolor{gray}{$40.8$} & \textcolor{gray}{$73.3$} & \textcolor{gray}{$18.2$} & \textcolor{gray}{$45.8$} \\
\textcolor{gray}{SimpleRL-Zero$^*$}                                 
  & \textcolor{gray}{$27.0$} & \textcolor{gray}{$6.8$}  & \textcolor{gray}{$54.9$} & \textcolor{gray}{$76.0$} & \textcolor{gray}{$25.0$} & \textcolor{gray}{$34.7$} & \textcolor{gray}{$37.4$} & \textcolor{gray}{$30.2$} & \textcolor{gray}{$23.2$} & \textcolor{gray}{$26.7$} \\
\textcolor{gray}{OpenReasoner-Zero$^*$}                             
  & \textcolor{gray}{$16.5$} & \textcolor{gray}{$15.0$} & \textcolor{gray}{$52.1$} & \textcolor{gray}{$82.4$} & \textcolor{gray}{$33.1$} & \textcolor{gray}{$47.1$} & \textcolor{gray}{$41.0$} & \textcolor{gray}{$66.2$} & \textcolor{gray}{$29.8$} & \textcolor{gray}{$48.0$} \\
\textcolor{gray}{Oat-Zero$^*$}                                  
  & \textcolor{gray}{$33.4$} & \textcolor{gray}{$11.9$} & \textcolor{gray}{$61.2$} & \textcolor{gray}{$78.0$} & \textcolor{gray}{$34.6$} & \textcolor{gray}{$43.4$} & \textcolor{gray}{$43.8$} & \textcolor{gray}{$70.1$} & \textcolor{gray}{$23.7$} & \textcolor{gray}{$46.9$} \\
\bottomrule
\end{tabular}
}
\end{table*}

\paragraph{Benchmarks}
We evaluate \mot on $6$ mathematical reasoning benchmarks: AIME 2024~\citep{li2024numinamath},  AIME 2025~\citep{li2024numinamath}, AMC~\citep{li2024numinamath}, MATH-500~\citep{hendrycks2021measuring},  Minerva~\citep{lewkowycz2022solving}, and OlympiadBench~\citep{he2024olympiadbench}.
AMC~\citep{li2024numinamath} comprises problems drawn from the AMC12 2022 and AMC12 2023 examinations.
Moreover, when employing Qwen2.5-Math-7B as the backbone, we further conduct evaluations on GPQA-Diamond~\citep{rein2024gpqa}, a challenging and high-quality subset of the Graduate-Level Google-Proof Question Answering benchmark, as well as on ARC-c~\citep{clark2018think}, an open-domain reasoning benchmark.

\paragraph{Evaluation Setup}
We set the maximum generation length to $8,192$ tokens, unless otherwise specified.
For the main experiments, following DeepSeek-R1~\citep{guo2025deepseek}, we adopt the \emph{Pass@k} evaluation protocol~\citep{chen2021evaluating} and report \emph{Pass@1} using non-zero temperature sampling.
To ensure a fair comparison with previous works~\citep{yan2025learning,fu2025srft}, we compute avg@32 for AIME 24, AIME 25, and AMC (avg@1 for others) using a temperature of $0.6$ and a top-$p$ value of $0.95$ for accuracy calculation.

\paragraph{Baselines}
Since \mot dynamically integrates GRPO~\citep{shao2024deepseekmath} and SFT, the most natural baselines are SFT and GRPO individually.
Furthermore, we compare \mot against the mix-policy approach LUFFY~\citep{yan2025learning}.
For experiments using Qwen2.5-Math-7B as the backbone, we additionally include SFT$\rightarrow$GRPO and SRFT\footnote{The results of SRFT are based on our own implementation, as the official code is not public.}~\citep{fu2025srft} as a baseline, as well as models trained with the Zero-RL procedure on the same backbone for a more comprehensive comparison. We also use PRIME-Zero \citep{cui2025process}, SimpleRL-Zero \citep{zeng2025simplerl}, OpenReasoner-Zero \citep{hu2025open} and Oat-Zero \citep{liu2025understanding} as baselines.

\paragraph{Implementation Details}
We apply GRPO~\citep{shao2024deepseekmath} as the RL algorithm to implement \mot.
We introduce a gating mechanism that adaptively assigns the coefficients $\alpha$ and $\beta$ to the RL loss and the SFT loss based on the rollout performance, respectively.
Formally, the gating mechanism is defined as:
\[
(\alpha, \beta) =
\begin{cases}
(0, 1), & \text{if } P \leq \gamma, \\[6pt]
(1, 0), & \text{if } P > \gamma,
\end{cases}
\]
where $P$ denotes model's performance as introduced in Section~\ref{sec:method} and $\gamma$ is the gate threshold.
We fix $\gamma$ at $0$ throughout all experiments on the Qwen Family models and $2/8$ for LLaMA, and provide relative ablation studies in Section~\ref{sec:gate_threshold_ablation}.
For hyperparameters, we use a constant learning rate of $5 \times 10^{-6}$ and adopt the AdamW optimizer for the policy model.
For rollout, we sample $8$ responses using a temperature of $1.0$.
The maximum generation length is set to $8,192$ tokens for all other models.
For other details that may not have been explicitly introduced, we have endeavored to follow previous works as closely as possible~\citep{zhao2025learning,zuo2025ttrl}.
All experiments were conducted on 8 x NVIDIA A800 80GB GPUs.

\subsection{Main Results}

Table~\ref{tab:Qwen7B-results} presents the overall performance of \mot on Qwen2.5-Math-7B.
As introduced in Section~\ref{sec:method}, in our implementation of \mot, the coefficients of the RL and SFT loss terms are both degraded and simplified into a binary form.
Despite this highly streamlined experimental setup, \mot still yields substantial performance gains.
It not only significantly outperforms both SFT-only and GRPO-only baselines, but also surpasses SFT$\rightarrow$GRPO, which requires substantially higher computational cost.
This suggests that simply concatenating the two training stages is not the most effective strategy.
Moreover, \mot achieves marked improvements over existing mixed-policy approaches such as LUFFY and SRFT, with particularly notable gains of 6.9 and 14.6 points on AIME 2024, respectively.
Furthermore, we conduct experiments on models of different scales and families to evaluate the effectiveness of \mot, including LLaMA3.1-8B and Qwen2.5-Math-1.5B, as shown in Table~\ref{tab:llama qwen results}.
Compared with SFT, GRPO, and LUFFY, \mot achieves substantial performance gains.

\begin{table*}[t]
\definecolor{bluishyellow}{rgb}{0.84, 0.92, 0.85}
\centering
\caption{Performance of \mot and baselines on LaMA3.1-8B and Qwen2.5-Math-1.5B. $^*$ means the results are taken from the LUFFY paper~\citep{yan2025learning}.}
\label{tab:llama qwen results}
\resizebox{.85\textwidth}{!}{%
\begin{tabular}{lccccccc}
\toprule
\textbf{Model} & \textbf{AIME 24} & \textbf{AIME 25} & \textbf{AMC} & \textbf{MATH-500} & \textbf{Minerva} & \textbf{Olympiad} & \textbf{Avg} \\
\midrule
LLaMA3.1-8B
  & $0.4$ & $0.1$  & $4.7$ & $13.8$ & $4.8$  & $3.9$ & $4.6$       \\
\midrule
\multicolumn{1}{r}{SFT$^*$}                                     
  & $0.5$ & $0.1$ & $5.4$ & $20.2$ & $4.0$ & $5.3$ & $5.9$ \\
\multicolumn{1}{r}{GRPO$^*$}                                  
  & $0.3$ & $0.5$ & $9.4$ & $23.4$ & $17.6$ & $6.1$ & $9.6$ \\
\multicolumn{1}{r}{LUFFY$^*$} & $1.9$ & $0.1$ & $13.5$ & $39.0$ & $15.1$ & $9.6$ & $13.2$ \\
\rowcolor{lightblue!100}\multicolumn{1}{r}{\mot} & $\mathbf{2.1}$ & $\mathbf{1.2}$ & $\mathbf{18.6}$ & $\mathbf{47.8}$ & $\mathbf{18.8}$ & $\mathbf{20.4}$ & $\mathbf{18.2}$ \\
\midrule
Qwen2.5-Math-1.5B
  & $2.8$ & $6.1$  & $24.5$ & $32.8$ & $11.0$  & $16.4$ & $15.6$      \\
\midrule
\multicolumn{1}{r}{SFT}                                     
  & $14.7$ & $17.6$ & $45.4$ & $78.4$ & $29.4$ & $35.7$ & $36.9$ \\
\multicolumn{1}{r}{GRPO}                                
  & $12.2$ & $8.5$ & $43.8$ & $71.0$ & $33.1$ & $35.3$ & $34.0$ \\
\multicolumn{1}{r}{LUFFY} & $14.1$ & $9.4$ & $43.5$ & $75.2$ & $26.1$ & $39.7$ & $34.7$ \\
\rowcolor{lightblue!100}\multicolumn{1}{r}{\mot} & $\mathbf{16.6}$ & $\mathbf{17.8}$ & $\mathbf{51.0}$ & $\mathbf{81.0}$ & $\mathbf{37.5}$ & $\mathbf{47.3}$ & $\mathbf{41.9}$ \\
\bottomrule
\end{tabular}
}
\end{table*}

\section{Empirical Analysis}
Our empirical analysis progressively reveals how \mot reconciles exploration and exploitation, stabilizes training, and ultimately enhances the reasoning ability.
We begin in $\S$~\ref{sec:explor_exploit} with an examination of \textcolor{cyan!40!black}{\textit{exploration}} and \textcolor{cyan!40!red}{\textit{exploitation}}.
In $\S$~\ref{sec:training_visualization}, we provide a training visualization, contrasting \mot with the conventional SFT$\rightarrow$GRPO.
Next, $\S$~\ref{sec:training_dynamics} investigates fine-grained training metrics of \mot.
Building on this, $\S$~\ref{sec:impact_of_off-policy_rl} explores the role of off-policy RL, testing whether alternative strategies for utilizing offline data yield benefits.
Finally, $\S$~\ref{sec:gate_threshold_ablation} presents a gate threshold ablation study.

\subsection{Exploration and Exploitation}
\label{sec:explor_exploit}

\mot inherently achieves an adaptive switching between RL and SFT. These two paradigms naturally correspond to the learning modes of \textcolor{cyan!40!black}{\textit{exploration}} and \textcolor{cyan!40!red}{\textit{exploitation}}.
Accordingly, we can examine whether \mot addresses the initial challenges from both perspectives.

\begin{figure}[!h]
    \centering
    \includegraphics[width=\textwidth]{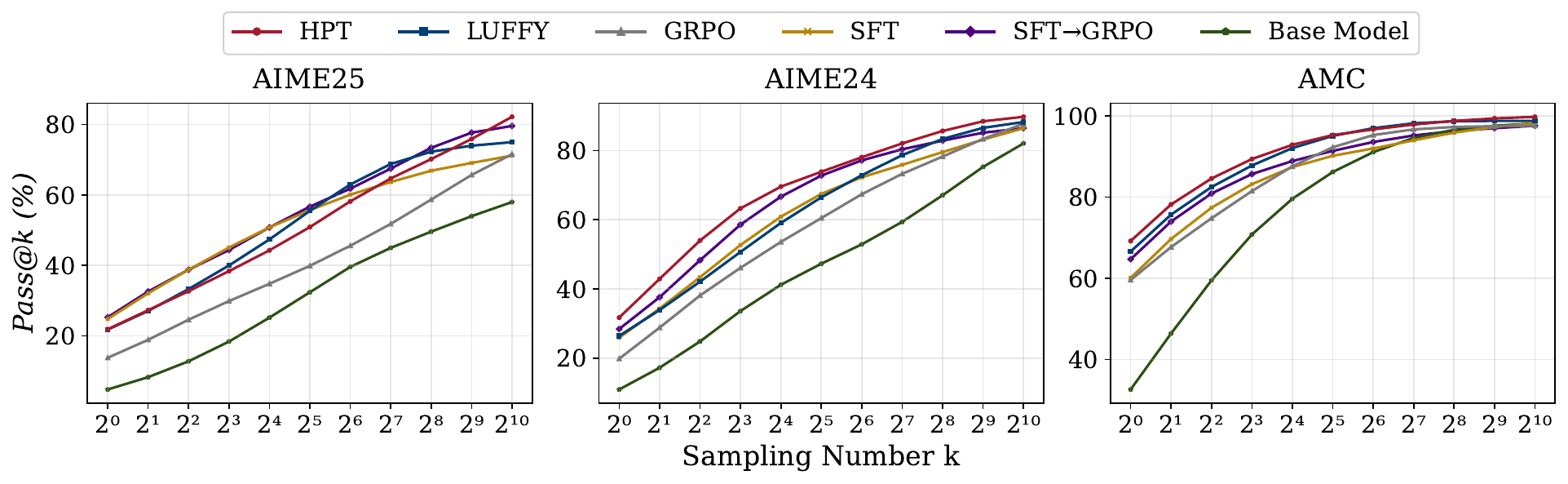}
    \caption{\emph{Pass@k} performance of \mot against baselines on Qwen2.5-Math-7B. The evaluation spans 3 benchmarks, with \emph{Pass@k} values estimated via bootstrap sampling from a set of $2048$ generated solutions per problem.
    }
    \label{fig:passk}
\end{figure}

\paragraph{\textcolor{cyan!40!black}{Exploration}}
From the exploration perspective, we want to analyze the model’s \emph{Pass@k} performance after training with \mot.
Recently, Limit-of-RLVR~\citep{yue2025does} demonstrated that while RLVR training yields a significant improvement in \emph{Pass@1}, it does not lead to gains in large-$k$ \emph{Pass@k}.
In other words, RLVR does not expand the capability boundary of the base model.
This finding has sparked broad discussions regarding the relationship between a model’s exploratory capacity and its \emph{Pass@k} performance. Moreover, \emph{Pass@k} has increasingly been recognized as a widely accepted metric for evaluating both the upper bound of model capability and its exploration ability.
We follow~\citet{yue2025does} to evaluate \emph{Pass@k} up to 1024 for each problem of AIME25, AIME24, and AMC for \emph{Pass@k} evaluation.
Based on these sets of generated solutions, we apply bootstrap sampling to obtain accurate estimates of \emph{Pass@k} scores for various values of $k$.
Figure~\ref{fig:passk} illustrates the resulting \emph{Pass@k} curves, comparing \mot against baselines and the base model.

\begin{itemize}[leftmargin=2em]
    \item 
    First, we can observe that methods incorporating SFT achieve higher large-$k$ \emph{Pass@k} compared to the GRPO (purely RL).
    This may be attributed to the introduction of data outside the model’s own distribution during SFT, which increases output uncertainty while also providing new knowledge from offline data, thereby enhancing the model’s exploratory capacity.
    \item 
    Furthermore, we identify an interesting phenomenon: since \mot dynamically integrates RL (GRPO) with SFT, we might intuitively expect its large-$k$ \emph{Pass@k} performance to fall between that of the two individual methods.
    However, \mot achieves the highest large-$k$ \emph{Pass@k} performance overall.
    \emph{This indicates that \method not only delivers substantial improvements in Pass@1, but also maximally preserves and enhances the model’s exploratory ability.}
\end{itemize}

\begin{table*}[!h]
\centering
\definecolor{darkgreen}{RGB}{50,100,0}
\definecolor{darkred}{RGB}{200,0,0}
\newcommand{\gain}[2]{\textcolor{darkred}{+#1}/\textcolor{darkgreen}{-#2}}
\newcommand{\gainp}[2]{\textcolor{darkred}{+#1\%}/\textcolor{darkgreen}{-#2\%}}

\caption{
Bidirectional analysis of exclusive solves on MATH-500, comparing the Qwen2.5-Math-7B trained with \mot against baseline methods (GRPO and LUFFY).
The notation \textcolor{darkgreen}{+X}\,/\,\textcolor{darkred}{-Y} in each cell indicates the performance trade-off: \textcolor{darkred}{+X} represents the number of problems solved by the \mot but not the baseline, while \textcolor{darkgreen}{-Y} represents the number solved by the baseline but not by the \mot.
}
\label{tab:mot-ex-solves}
\resizebox{\textwidth}{!}{
  \begin{tabular}{@{}lcccccc@{}}
  \toprule
  Methods & Level 1 & Level 2 & Level 3 & Level 4 & Level 5 & Overall \\
  & (N=43) & (N=90) & (N=105) & (N=128) & (N=134) & (N=500) \\
  \midrule
  \textbf{GRPO}\\
  \quad Absolute & \gain{0}{0} & \gain{5}{1} & \gain{9}{2} & \gain{17}{4} & \gain{27}{8} & \gain{58}{15} \\
  \quad Percentage & \gainp{0.0}{0.0} & \gainp{5.6}{1.1} & \gainp{8.6}{1.9} & \gainp{13.3}{3.1} & \gainp{20.1}{6.0} & \gainp{11.6}{3.0} \\
  \textbf{LUFFY}\\
  \quad Absolute & \gain{1}{0} & \gain{5}{1} & \gain{5}{3} & \gain{10}{5} & \gain{22}{7} & \gain{43}{16} \\
  \quad Percentage & \gainp{2.3}{0.0} & \gainp{5.6}{1.1} & \gainp{4.8}{2.9} & \gainp{7.8}{3.9} & \gainp{16.4}{5.2} & \gainp{8.6}{3.2} \\
  \bottomrule
  \end{tabular}
}
\end{table*}

\paragraph{\textcolor{cyan!40!red}{Exploitation}}
From the exploitation perspective, the key question is whether our method, by leveraging SFT, enhances the model’s initial competence and facilitates subsequent RL training.
As illustrated in Figure~\ref{fig:SFT+GRPO-3-5-ep50}, RL training alone may fail to solve many problems (white line), requiring the dynamic intervention of SFT.
To investigate this, we analyze its exclusive solves against the GRPO and LUFFY, building upon the results from the evaluation on MATH-500 with Qwen2.5-Math-7B as the backbone, as shown in Table~\ref{tab:mot-ex-solves}.
The red numbers denote problems that are solved by our method but not by GRPO or LUFFY, i.e., problems newly acquired through our training procedure.
Three clear trends emerge from the analysis:
\begin{itemize}[leftmargin=2em]
    \item First, the red counts consistently increase with problem difficulty, suggesting that \mot improves the model’s ability to tackle more challenging problems.
    \item Second, the green counts within the red boxes remain essentially unchanged across settings: this indicates that, compared with existing methods, \mot preserves performance on problems that the model could already solve, thereby mitigating the risk of catastrophic forgetting.
    \item Finally, the fact that the red counts are consistently large relative to both baselines demonstrates that our method enables the model to acquire a substantial number of problems that prior approaches struggled to solve.
\end{itemize}

\subsection{Training Visualization}
\label{sec:training_visualization}

\begin{figure}[!t]
    \centering
    \includegraphics[width=\linewidth]{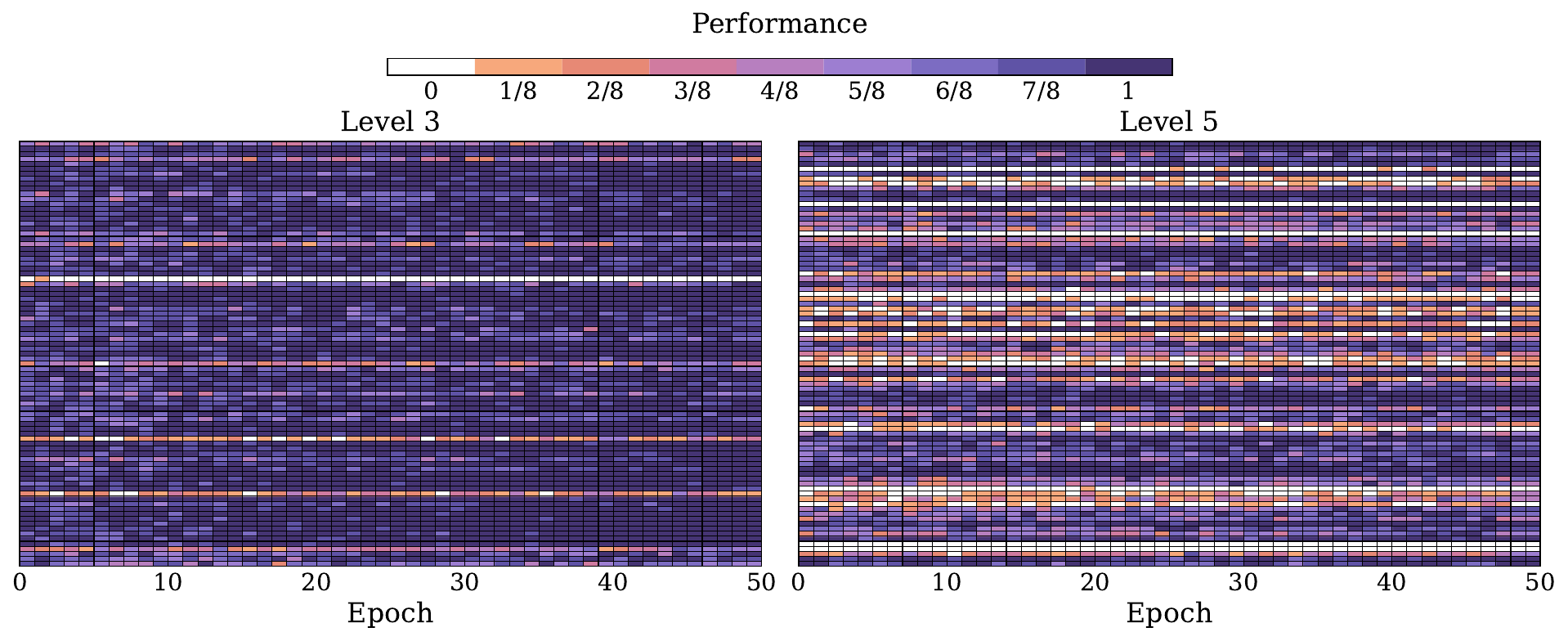}
    \caption{GRPO training dynamics of SFT$\rightarrow$GRPO on Qwen2.5-Math-1.5B across 50 training epochs. We visualize the model's per-question sampling accuracy throughout the training process.}
    \label{fig:SFT+GRPO-3-5-ep50}
    \vspace{-2mm}
\end{figure}

\begin{figure}[!t]
    \centering
    \par \vspace{-1.5mm}
    \includegraphics[width=\linewidth]{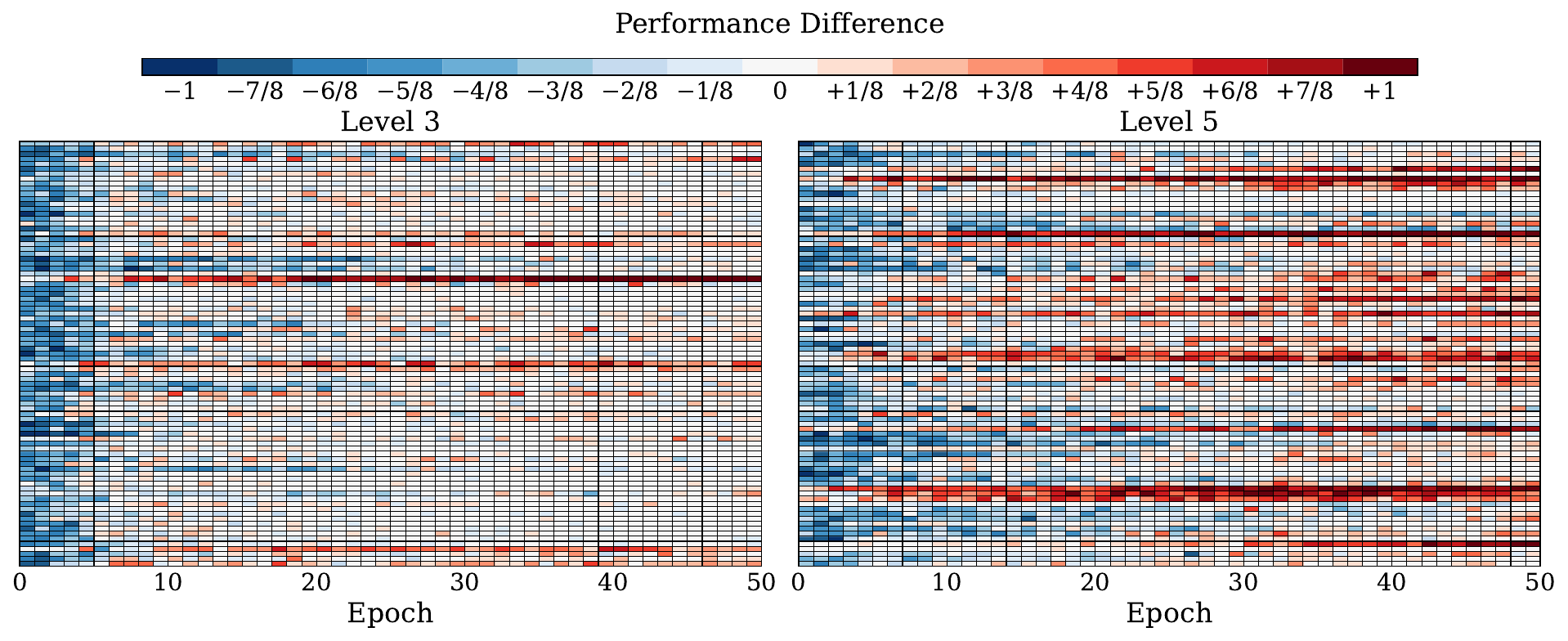}
    \caption{Performance difference (\mot v.s. SFT$\rightarrow$GRPO) on Qwen2.5-Math-1.5B across 50 training epochs.
    A diverging color scale indicates the advantage: red for \mot, blue for SFT$\rightarrow$GRPO, and white for no difference.}
    \label{fig:HPT-SFT+GRPO-3-5-ep50}
\end{figure}

To facilitate a fine-grained examination of the training process and thereby obtain deeper insights into how \mot works, we conduct a visualization analysis comparing the SFT$\rightarrow$GRPO approach with \mot.
We sample 255 problems from the MATH dataset~\citep{hendrycks2021measuring} for subsequent training, with 85 problems each from Levels 3, 4, and 5.
For SFT$\rightarrow$GRPO, we perform 50 epochs of GRPO on a Qwen2.5-Math-1.5B model fine-tuned with SFT, tracking rollout accuracy across training, as shown in Figure~\ref{fig:SFT+GRPO-3-5-ep50}.
We track the rollout accuracy for each sample throughout the entire training process.
To highlight difficulty effects, we focus on Levels 3 and 5 as representative cases.
The left subplot shows Level 3 (easier) problems, and the right shows Level 5 (hardest).
Notably, GRPO frequently produces dense white regions, and sometimes even continuous white lines, reflecting widespread rollout errors across outputs.
This illustrates a core limitation of RL methods: they struggle to learn effectively when frequent rollout errors occur across all outputs.

In parallel, we train Qwen2.5-Math-1.5B from scratch for 50 epochs to visualize \mot.
To compare against SFT$\rightarrow$GRPO and enable a more intuitive comparison, we conduct a differential analysis of the training dynamics.
Specifically, we calculate the accuracy difference at corresponding positions (at matched prompts and steps) in the evaluation grid between two methods: red indicates \mot is better, blue the opposite.
Figure~\ref{fig:HPT-SFT+GRPO-3-5-ep50} presents the results of the difference plots.
Notably, SFT$\rightarrow$GRPO actually requires greater computational resources than \mot: it involves a preceding SFT phase, and our approach also reduces computational costs during the transition from GRPO to SFT, as expensive operations such as rollouts are no longer required.
This unfair comparison leads to an initial dominance of the blue regions, which is expected since the SFT stage in SFT$\rightarrow$GRPO has already incorporated substantial prior knowledge.
However, in the later stages of training, \mot still surpasses and ultimately reveals the dominance of the red regions, indicating that \mot consistently outperforms SFT$\rightarrow$GRPO by substantially enhancing learning performance on the training set.
This advantage becomes even more pronounced in the Level 5 subplot, suggesting that \mot provides particular benefits for learning on more challenging problems, which may be attributed to its use of question-level rollout performance as feedback.

\subsection{Training Dynamics}
\label{sec:training_dynamics}
In this section, we investigate the training dynamics of \mot, focusing on validation performance, entropy, response length, and the offline data ratio.
Our analysis centers on two aspects: whether \mot enables the model to acquire knowledge from offline data when its initial capabilities are limited, and whether its performance can be further enhanced through continued exploration with reinforcement learning.

\begin{figure}[!h]
  \centering
  \includegraphics[width=1.0\textwidth]{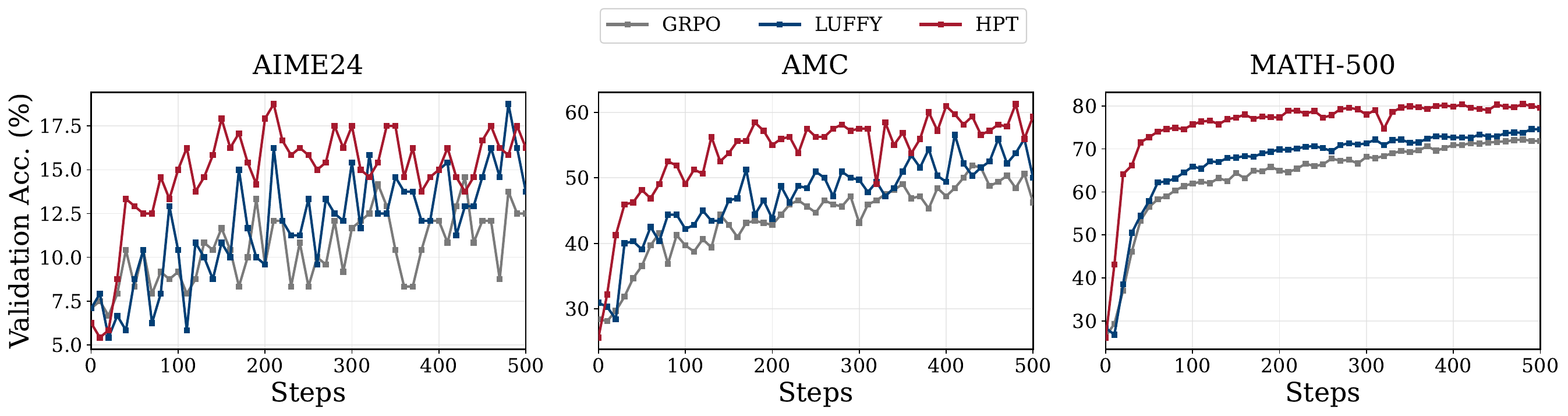}
  \caption{Validation performance comparisons on Qwen2.5-Math-1.5B across benchmarks.}
  \label{fig:val_score_1.5B}
\end{figure}

\paragraph{Validation Performance.}
We track the validation performance on the Qwen2.5-Math-1.5B as shown in Figure~\ref{fig:val_score_1.5B}, where \mot consistently outperforms the baselines and delivers stable improvements across multiple benchmarks.

\begin{wrapfigure}{r}{.5\textwidth}
    \centering
    \includegraphics[width=\linewidth]{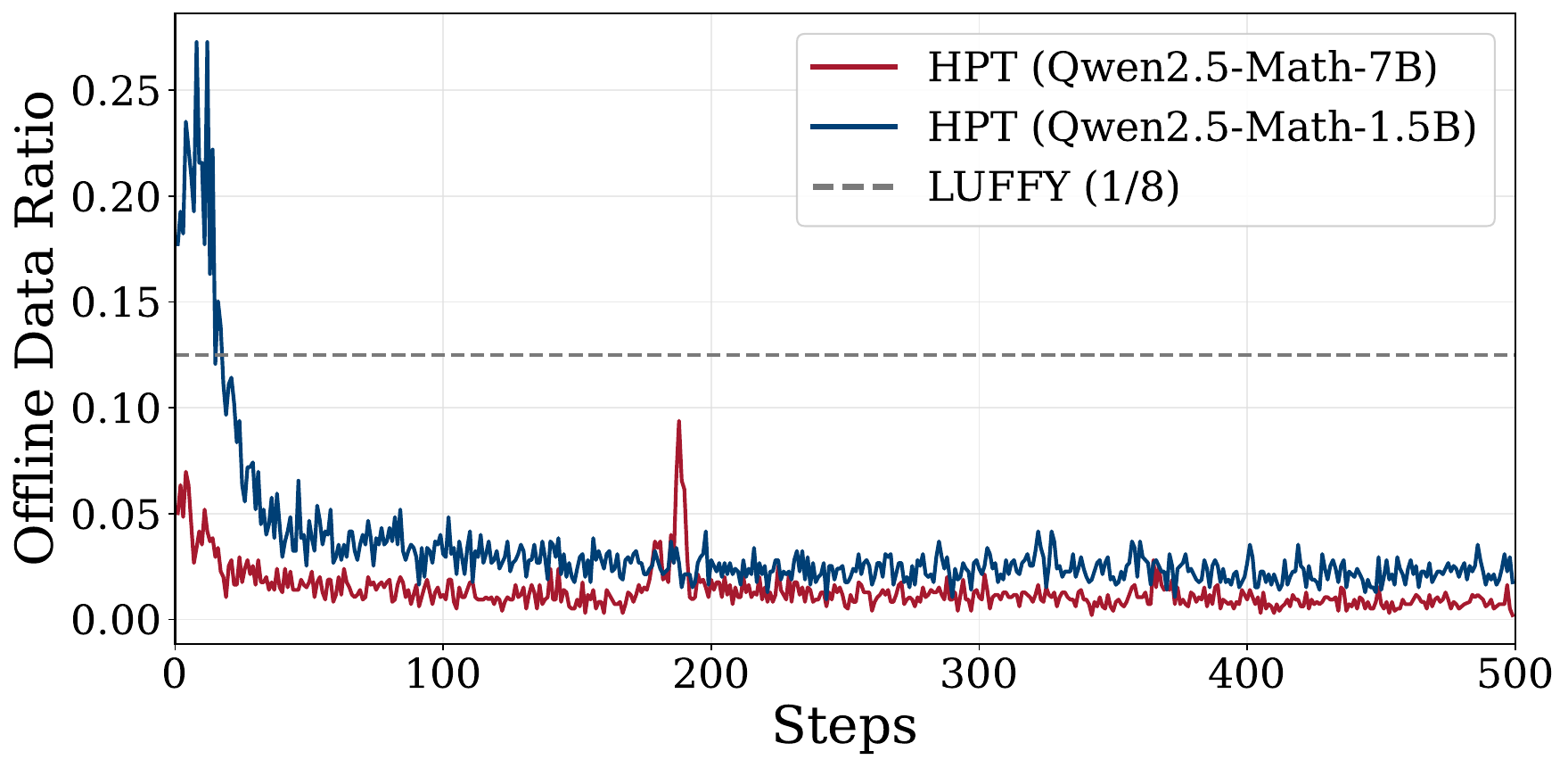}
    \caption{Dynamic offline data ratio dynamics during training.
    The offline data ratio is calculated as the proportion of offline training samples relative to the total training data at each step.
    }
    \label{fig:training_ratio}
\end{wrapfigure}
\paragraph{Offline Data Ratio.}
We begin by quantifying the fraction of prompts whose gradients update the model through the SFT loss versus the RL loss at each training step, as shown in Figure~\ref{fig:training_ratio}.
The offline data ratio is defined as the proportion of offline samples relative to the total number of training samples in each batch, with online samples calculated based on the remaining batch capacity.
As expected, when the model has not yet acquired competence on the target tasks, the early phase is characterized by a large proportion of SFT-driven updates.
As training progresses and the model’s on-policy reward increases, the mixture gradually shifts: the contribution of RL grows while that of SFT diminishes, eventually stabilizing at a small but non-zero level.
This trend is observed for both Qwen2.5-Math-7B and Qwen2.5-Math-1.5B.
The weaker 1.5B model remains in the SFT-dominated regime for a longer period before transitioning, whereas the stronger 7B model shifts earlier.
These results align with our technical analysis of the design of~\mot, where the mixing ratio is automatically adjusted based on performance rather than fixed in advance like LUFFY.

\begin{figure}[!h]
  \centering
  \includegraphics[width=1.0\textwidth]{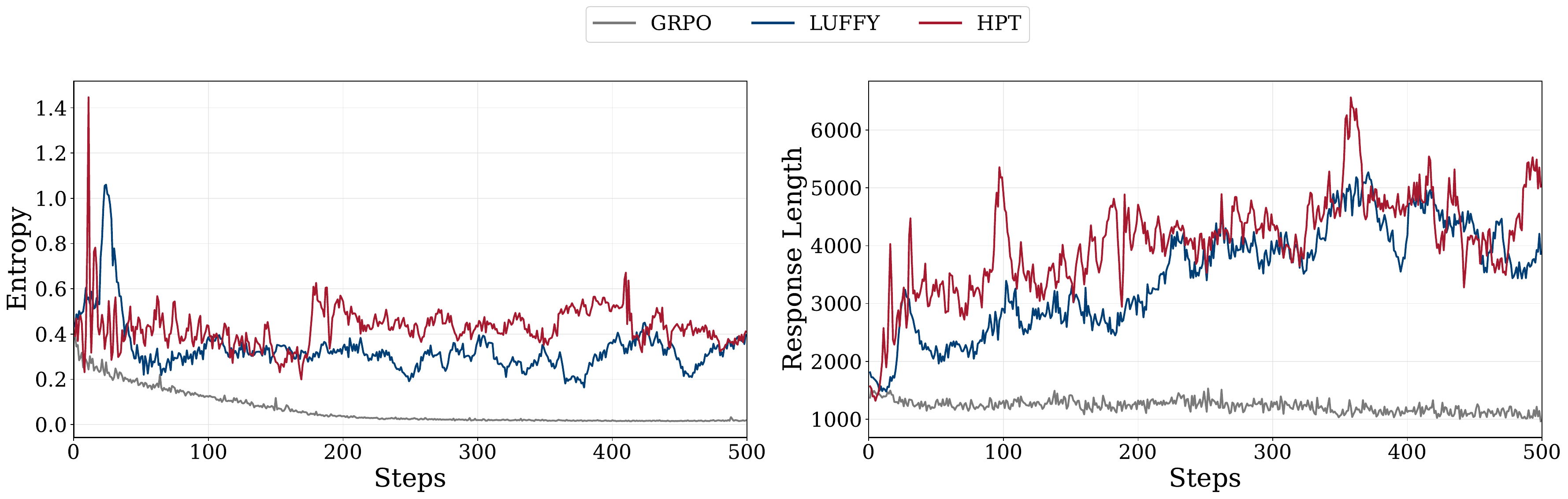}
  \caption{Comparisons of training dynamics across different methods: (left) The entropy measures the diversity of model outputs, indicating exploration behavior; (right) The response length tracks the average length of generated responses.}
  \label{fig:training_metrics_7B}
\end{figure}
\paragraph{Entropy and exploration.}
Figure~\ref{fig:training_metrics_7B} (left) tracks token-level entropy over 500 steps. \mot maintains higher entropy than GRPO throughout the training phases.
This is expected as the offline SFT trajectories are derived from the external demonstration distribution, which consequently increases the diversity in the model’s outputs.

\paragraph{Response length and acquired reasoning patterns.}
Figure~\ref{fig:training_metrics_7B} (right) reports the average response length. Our offline SFT trajectories have a length of up to 8k tokens. Under \mot, the model’s response length increases quickly during the early steps but does not jump to the 8k ceiling.
More importantly, after the method shifts toward RL and the SFT proportion plateaus at a low level, the response length does not regress. This persistence suggests that the model has internalized long-form reasoning routines from the offline data rather than merely echoing teacher outputs. In other words, the learned reasoning pattern becomes part of the policy, and RL fine-tuning refines it instead of erasing it.

\subsection{Impact of Off-policy RL}
\label{sec:impact_of_off-policy_rl}

\begin{table*}[!h]
\definecolor{bluishyellow}{rgb}{0.84, 0.92, 0.85}
\definecolor{bluishcyan}{rgb}{0.74, 0.93, 0.95}
\centering
\caption{Performance of different training paradigms to evaluate the impact of Off-policy RL.
\textbf{SFT/ON} denotes SFT/On-policy~(\mot), 
\textbf{OFF/ON} denotes Off-policy/On-policy,
and \textbf{Mix/ON} denotes Mix-policy/On-policy.}
\label{tab:offpolicy_rl_impact}
\resizebox{.85\textwidth}{!}{%
\begin{tabular}{lccccccc}
\toprule
\textbf{Name} & \textbf{AIME 24} & \textbf{AIME 25} & \textbf{AMC} & \textbf{MATH-500} & \textbf{Minerva} & \textbf{Olympiad} & \textbf{Avg} \\
\midrule
\multicolumn{1}{r}{OFF/ON}                                     
  & $16.6$ & $11.8$ & $47.3$ & $76.2$ & $35.3$ & $41.6$ & $38.1$ \\
\multicolumn{1}{r}{Mix/ON}                                  
  & $\mathbf{16.7}$ & $17.2$ & $46.9$ & $79.4$ & $37.5$ & $43.9$ & $40.3$ \\
\multicolumn{1}{r}{SFT/ON}
  & $16.6$ & $\mathbf{17.8}$ & $\mathbf{51.0}$ & $\mathbf{81.0}$ & $\mathbf{37.5}$ & $\mathbf{47.3}$ & $\mathbf{41.9}$ \\
\bottomrule
\end{tabular}%
}
\end{table*}

In our work, we have only made preliminary attempts at unifying post-training by integrating RL with SFT.
However, off-policy RL represents an important training paradigm that emphasizes leveraging offline data. To this end, we further conduct experiments to investigate its influence and potential role.

We compare three different training paradigms:
(1) SFT/On-policy, the model alternates between SFT and on-policy RL, which corresponds to the method we introduced above (\mot);
(2) Off-policy/On-policy, the model alternates between off-policy RL and on-policy RL during training;
and (3) Mix-policy/On-policy, the model combines the loss from SFT and off-policy RL, and dynamically switches it with the on-policy RL objective. For the Mix setting, we performed hyperparameter search and found the optimal SFT/OFF weighting ratio to be $1/10$, i.e., the coefficients of the SFT loss and the off-policy loss are set to $0.1$ and $1.0$, respectively.
We replicate the off-policy RL implementation described in LUFFY~\citep{yan2025learning}, and all experiments are conducted in the same settings to ensure fairness.

We evaluate the results of three methods on six math benchmarks. Table \ref{tab:offpolicy_rl_impact} presents results.
Overall, the SFT/ON method achieves the best average performance (41.9), outperforming both Mix/ON (40.3) and OFF/ON (38.1).
This suggests that off-policy RL may not be essential, as SFT already serves effectively as the training method of \mot for learning from offline data.

\subsection{Gate Threshold Ablation}
\label{sec:gate_threshold_ablation}

\begin{figure}[!h]
    \centering
    \includegraphics[width=\linewidth]{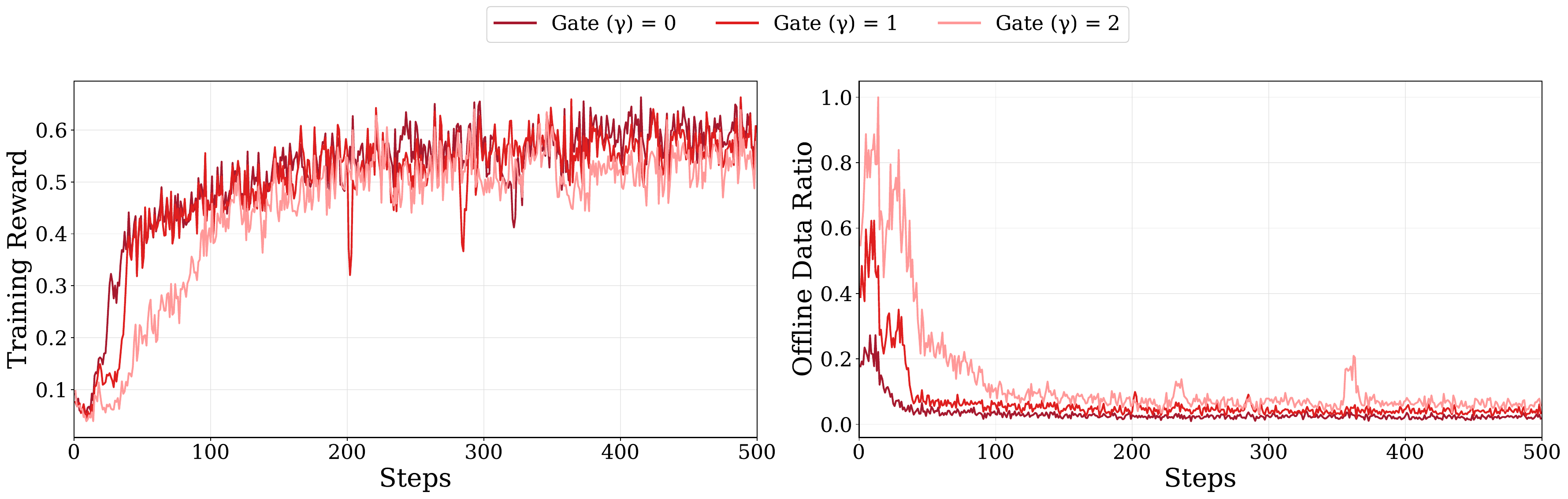}
    \caption{
    Training reward (left) and offline data ratio (right) comparisons across different gate settings on Qwen2.5-Math-1.5B.
    }
    \label{fig:gate_training_reward}
\end{figure}

In this section, we investigate the effect of different gate thresholds $\gamma$. A value of $\gamma=0$ indicates that the model switches to SFT only when it fails all questions.
Similarly, $\gamma=1/8$ and $\gamma=2/8$ correspond to settings where the model remains in on-policy reinforcement learning as long as it answers at least one or two out of eight questions correctly, respectively.
To visualize the impact of the gating mechanism, we conduct experiments on the Qwen2.5-Math-1.5B under three different gate settings.
As shown in Figure~\ref{fig:gate_training_reward}, we analyze the training dynamics by tracking the dynamics of rewards and the proportion of offline data utilized throughout training, thereby highlighting how different gate thresholds mediate the balance between leveraging offline demonstrations and incorporating online feedback.
We observe that, under different gate thresholds, varying degrees of engagement with offline data–based SFT learning emerge.
A larger gate threshold introduces a greater extent of SFT based on offline data, as expected.

\begin{table*}[!ht]
\definecolor{bluishyellow}{rgb}{0.84, 0.92, 0.85}
\definecolor{bluishcyan}{rgb}{0.74, 0.93, 0.95}
\centering
\caption{Performance of \mot with different switch gate $\gamma$ on Qwen2.5-Math-1.5B.}
\label{tab:gate_results}
\resizebox{.85\textwidth}{!}{%
\begin{tabular}{lccccccc}
\toprule
\textbf{Name} & \textbf{AIME 24} & \textbf{AIME 25} & \textbf{AMC} & \textbf{MATH-500} & \textbf{Minerva} & \textbf{Olympiad} & \textbf{Avg} \\
\midrule
\multicolumn{1}{r}{$\gamma=2/8$} & $15.8$ & $13.0$ & $49.0$ & $77.6$ & $34.6$ & $44.1$ & $39.0$ \\
\multicolumn{1}{r}{$\gamma=1/8$} & $\mathbf{18.1}$ & $14.2$ & $46.0$ & $75.4$ & $35.7$ & $42.5$ & $38.7$ \\
\multicolumn{1}{r}{$\gamma=0$} & $16.6$ & $\mathbf{17.8}$ & $\mathbf{51.0}$ & $\mathbf{81.0}$ & $\mathbf{37.5}$ & $\mathbf{47.3}$ & $\mathbf{41.9}$ \\
\bottomrule
\end{tabular}%
}
\end{table*}

To further compare the performance across different gating strategies, we evaluate the three trained models on six benchmarks.
Table~\ref{tab:gate_results} presents the results.
Among the three configurations, $\gamma=0$ achieves the best overall performance with an average score of 41.9, outperforming both $\gamma=1/8$ (38.7) and $\gamma=2/8$ (39.0).
This observation suggests that simply incorporating more SFT does not necessarily lead to better outcomes.
Instead, it is crucial to maintain a dynamic balance between the \textcolor{cyan!40!black}{exploration} of RL and the \textcolor{cyan!40!red}{exploitation} of SFT.
The optimal degree of this gating mechanism should be adjusted according to the characteristics of the base model and the specific training data employed.

\section{Conclusion}

In this paper, we introduce the Unified Policy Gradient Estimator to provide a theoretical framework for LLM post-training. We demonstrate that SFT and RL optimize a common objective, with their respective gradients representing different bias-variance tradeoffs. Motivated by this unified perspective, we propose Hybrid Post-Training (HPT), an algorithm that dynamically adapts between SFT for exploitation and RL for exploration based on real-time performance feedback. Extensive empirical validation shows that HPT consistently outperforms strong baselines, including sequential and static mixed-policy methods, across various models and benchmarks. Our work contributes both a unifying theoretical perspective on post-training and a practical algorithm that effectively balances exploitation and exploration to enhance model capabilities.

\bibliography{colm2025_conference}
\bibliographystyle{colm2025_conference}

\newpage
\appendix

\section{Gradient Derivation for Classical Algorithms} \label{sec:appdendix_grad_cal}
\subsection{Gradient of SFT}

We first consider the SFT process as a warm-up. 
As mentioned in the previous section, SFT takes a pre-trained foundation model and further makes the model more specialized by training its output prediction distribution to align with domain-specific data. 
The fine-tuning process uses the same cross-entropy loss as in model pre-training, defined as follows,

\begin{equation}\label{eq_sft_loss}
    \mathcal{L}_{SFT}(\theta) = - \sum_{i=1}^{N} \sum_{t = 1}^{|\tau_{i}|}\log \pi_{\theta} (\tau_{i,t} |q_i, \tau_{i, <t} ).
\end{equation}
where $\mathcal{D}_{SFT} = \{(q_i, \tau_i)\}_{i \in [N]}$ denotes the SFT dataset consisting of $N$ question and trajectory pairs. $\tau_t$ denotes the $t$-th token in the trajectory and $\tau_{<t}$ denotes all the tokens prior to $\tau_t$.

For any $t$, the LLM outputs the next-token prediction as a probability distribution. In the context of RL, such a probability distribution has been commonly considered as a stochastic policy. 
Then, the gradient calculation of SFT can be obtained by directly taking the derivative of Equation \eqref{eq_sft_loss} and takes the following form:
\begin{equation}\label{eq_SFT_grad}
    \nabla \mathcal{J}_{SFT}(\theta) = -\nabla \mathcal{L}_{SFT}(\theta) = 
     \sum_{i=1}^{N} \sum_{t = 1}^{|\tau_{i}|}
    \nabla \pi_\theta (\tau_{i,t} |q_i, \tau_{i, <t} ) \frac{1}{\pi_\theta (\tau_{i,t} |q_i, \tau_{i, <t} )}.
\end{equation}

In this section, we slightly abuse the notion of policy gradient and consider the SFT as a case of behavioral cloning (BC) \citep{torabi2018behavioral}, and Equation \eqref{eq_SFT_grad} can be seen as a specific form of policy gradient.

\subsection{Gradient of Online RL: PPO, GRPO and Beyond}
For online RL, we first consider Proximal Policy Optimization (PPO) \citep{schulman2017proximal} and a series of its derivations. PPO is a pivotal technique for RLVR in LLMs. Motivated by TRPO, PPO keeps the new policy close to the old policy, and perform conservative policy updates by incorporating a clipped version of its policy ratio in its objective. The clipping function was shown to stabilize the training process and avoid performance collapse during training. In this section, we omit the regularization terms, such as the KL divergence and entropy. The loss objective for PPO can be written as follows,
\begin{equation}\label{eq_PPO_loss}
    \cL_{PPO}(\pi_\theta) = 
    - \frac{1}{N}\sum_{i=1}^N \frac{1}{G}\sum_{j=1}^G \frac{1}{|\tau_j|} \sum_{t=1}^{|\tau_j|} 
        \min (r_{i,j,t}(\theta) \hat{A}_{i,j}, \text{clip}(r_{i,j,t}(\theta), 1-\epsilon, 1+\epsilon)  \hat{A}_{i,j}),
\end{equation}
In this setting, we consider questions sampled from a given dataset $\Dc_{RL} \triangleq \{q_i\}_{i=1}^N$, and for each question, we consider $G$ trajectories independently sampled using a reference policy $\pi_{ref}$.
We use $r_{i,j,t}(\theta) = \frac{\pi_\theta (\tau_{i,j,t} |q_i, \tau_{i,j, <t} ) }{\pi_{ref} (\tau_{i,j,t} |q_i, \tau_{i,j, <t} )}$ to denote the policy ratio $\pi_\theta / \pi_{ref}$ introduced for importance sampling, $\epsilon$ denotes the clipping factor for the importance sampling ratio, enhancing stability.

For PPO, $\hat{A}$ is estimated using the Generalized Advantage Estimation (GAE) \citep{schulman2015high}, calculated based on the reward of the sampled trajectories. For the case of GRPO, the advantage estimate $\hat{A}$ is calculated based on a set of sampled trajectories. Given question $q_i$, a group of sampled roll-out trajectoried $\{\tau_{i,j}\}_{j \in [G]}$ with verifiable reward $R(\tau_{i,j}) \in \{0, 1\}$, $\hat{A}_{i,j}$ is calculated as the normalized reward over the group.

\begin{equation}\label{eq_A_GRPO}
    \hat{A}_{i,j} = \frac{R(\tau_{i,j}) - \text{mean} (\{R(\tau_{i,k})\}_{k \in [G]})}{ \text{std} (\{R(\tau_{i,k})\}_{k \in [G]})},
\end{equation}

Compared to PPO, the most significant difference introduced by GRPO is the group relative advantage described above. Notably, the original manuscript of GRPO has also induced a sequence-level policy gradient balancing and a KL regularization term. However, more recent works such as \citep{yu2025dapo} have removed or modified these terms in general.

The clipped surrogate objective in PPO and similar algorithms enhances the stability of the RL training process by turning off gradient propagation on samples where $\pi_\theta$ moves too far from $\pi_{ref}$. For gradient calculation, this can be represented as an indicator function $\one_{clip}$. 

\begin{equation}\label{eq_PPO_grad}
    \nabla \mathcal{J}_{PPO} = -\nabla \cL_{PPO} =     \frac{1}{N}\sum_{i=1}^N \frac{1}{G}\sum_{j=1}^G \frac{1}{|\tau_j|} \sum_{t=1}^{|\tau_j|} 
    \nabla \pi_\theta (\tau_{i,j,t} |q_i, \tau_{i,j, <t} )  \frac{\hat{A}_{i,j} \one_{clip}}{\pi_{ref} (\tau_{i,j,t} |q_i, \tau_{i,j, <t} )} .
\end{equation}

Apart from PPO and GRPO, many recent RL algorithms for RL post-training in LLMs can be shown to exhibit a similar form for their policy gradient calculations.

\subsection{Gradient of Offline RL}

As stated in the previous sections, many recent studies seek to leverage offline data in the online RL training process for LLMs.
These methods consider expert demonstration data as trajectories sampled from a near-optimal policy, and perform RL updates on these data based on policy gradient updates. These algorithms are adapted from the online RL literature and often combine offline and online training, setting them apart from simple SFT.

Taking SRFT \citep{fu2025srft} as an instance, the offline RL objective can be written as follows

\begin{equation}\label{eq_SRFT_loss}
    \cL_{SRFT}(\pi_\theta) = 
    - \frac{1}{N}\sum_{i=1}^N \frac{1}{G}\sum_{j=1}^G \frac{1}{|\tau_j|} \sum_{t=1}^{|\tau_j|} 
        \pi_\theta (\tau_{i,j,t} |q_i, \tau_{i,j, <t} )  \hat{A}_{i,j},
\end{equation}

This objective is derived from the GRPO objective in Equation \eqref{eq_PPO_loss}, while setting $\pi_{ref} \equiv 1$ and removing the clipping mechanism since it becomes imbalanced.
The motivation behind setting $\pi_{ref} \equiv 1$ is that $\pi_{ref}$ is typically unavailable for offline data. Under the assumption that the demonstration policy evenly covers the current policy $\pi_\theta$. In this case, setting $\pi_{ref}$ to 1 changes the algorithm from importance sampling to rejection sampling.
The policy gradient of the offline SRFT objective can be derived consequently.
\begin{equation}\label{eq_SRFT_grad}
    \nabla \mathcal{J}_{SRFT} = -\nabla \cL_{SRFT} =     \frac{1}{N}\sum_{i=1}^N \frac{1}{G}\sum_{j=1}^G \frac{1}{|\tau_j|} \sum_{t=1}^{|\tau_j|} 
    \nabla \pi_\theta (\tau_{i,j,t} |q_i, \tau_{i,j, <t} )  \frac{\hat{A}_{i,j} }{\pi_{ref} = 1} .
\end{equation}

\section{Additional Theoretical Details for Section~\ref{sec:shared-common-objective}}
\label{app:theory-details}

\newtheorem{propositionA}{Proposition}
\renewcommand{\thepropositionA}{A\arabic{propositionA}}
\newtheorem{theoremA}{Theorem}
\renewcommand{\thetheoremA}{A\arabic{theoremA}}
\newtheorem{lemmaA}{Lemma}
\renewcommand{\thelemmaA}{A\arabic{lemmaA}}
\newtheorem{corollaryA}{Corollary}
\renewcommand{\thecorollaryA}{A\arabic{corollaryA}}
\newtheorem{remarkA}{Remark}
\renewcommand{\theremarkA}{A\arabic{remarkA}}

\subsection{Deriving Equation~\ref{eq:master_grad_pi_measure} from Equation~\ref{eq:master_obj}}
\label{app:derivation}

\begin{lemmaA}[Score-function identity]
\label{lemA:score}
For density $\pi_\theta$ and integrable $f(\tau)$,
\[
\nabla_\theta \,\E_{\tau\sim \pi_\theta}[f(\tau)]
=\E_{\tau\sim \pi_\theta}\!\big[f(\tau)\,\nabla_\theta \log \pi_\theta(\tau)\big],
\qquad
\E_{\tau\sim \pi_\theta}\!\big[\nabla_\theta \log \pi_\theta(\tau)\big]=0.
\]
\end{lemmaA}

\begin{lemmaA}[Differentiating an expectation with parameterized integrand]
\label{lemA:param-int}
For differentiable $f_\theta$,
\[
\nabla_\theta \,\E_{\tau\sim \pi_\theta}[f_\theta(\tau)]
= \E_{\tau\sim \pi_\theta}\!\big[\nabla_\theta \log \pi_\theta(\tau)\,f_\theta(\tau)+\nabla_\theta f_\theta(\tau)\big].
\]
\end{lemmaA}

\begin{lemmaA}[Measure-change (importance reweighting) identity]
\label{lemA:measure}
Let $s(\tau\mid q)$ be any sampling density that is positive wherever $\pi_\theta(\tau\mid q)$ is. Then
\[
\E_{\tau\sim \pi_\theta}\!\big[f(\tau)\,\nabla_\theta \log \pi_\theta(\tau)\big]
=
\E_{\tau\sim s}\!\Big[\frac{\pi_\theta(\tau)}{s(\tau)}\,f(\tau)\,\nabla_\theta \log \pi_\theta(\tau)\Big]
=
\E_{\tau\sim s}\!\Big[\frac{1}{s(\tau)}\,f(\tau)\,\nabla_\theta \pi_\theta(\tau)\Big].
\]
\end{lemmaA}

\begin{proof}
By Lemma~\ref{lemA:score},
$\nabla \E_{\pi_\theta}[r(\cdot\mid q)]=\E_{\pi_\theta}[r(\cdot\mid q)\,\nabla \log \pi_\theta]$.
For the data-adherence term, since
$\mathrm{KL}(\pi_\beta\|\pi_\theta)=\E_{\pi_\beta}[\log \pi_\beta-\log \pi_\theta]$ and $\pi_\beta$ does not depend on $\theta$,
we have $-\mu\,\nabla \mathrm{KL}(\pi_\beta\|\pi_\theta)=\mu\,\E_{\pi_\beta}[\nabla \log \pi_\theta]$.
Summing yields the claim.
\end{proof}

\subsection{Extension: Adding a Trust-Region Regularizer}
\label{app:lambda-extension}
A trust region encourages conservative policy updates by penalizing the KL divergence from the current policy $\pi_\theta$ to a fixed reference policy $\pi_{ref}$:
\[
\lambda\,\mathrm{KL}\big(\pi_\theta(\cdot\mid q)\,\|\,\pi_{ref}(\cdot\mid q)\big),\qquad \lambda\ge 0.
\]
It is the penalty form of the constrained problem
\[
\max_\theta\ \E_{\tau\sim \pi_\theta}[r(\tau\mid q)]
\quad \text{s.t.}\quad \mathrm{KL}\big(\pi_\theta\|\pi_{ref}\big)\le \delta,
\]
where $\lambda$ acts as the Lagrange multiplier tied to the trust-region radius $\delta$.
Typical choices are $\pi_{ref}=\pi_{\theta_{old}}$ (on-policy stability, TRPO/PPO-style). This penalty controls step sizes, dampens distribution shift, and yields clipping-style masks when optimized with PPO surrogates.

\medskip
\noindent\textbf{Objective and gradient with trust region.}
Augmenting the Common Objective with the trust-region term gives
\[
\widetilde{\mathcal{J}}_{\lambda,\mu}(\theta)
=
\E_{\tau\sim \pi_\theta(\cdot\mid q)}[r(\tau\mid q)]
\;-\;
\lambda\,\mathrm{KL}\!\big(\pi_\theta(\cdot\mid q)\,\|\,\pi_{ref}(\cdot\mid q)\big)
\;-\;
\mu\,\mathrm{KL}\!\big(\pi_\beta(\cdot\mid q)\,\|\,\pi_\theta(\cdot\mid q)\big),
\]
whose gradient is
\[
\nabla_\theta \widetilde{\mathcal{J}}_{\lambda,\mu}(\theta)
=
\E_{\tau\sim \pi_\theta}\!\Big[\big(r(\tau\mid q)-\lambda \log\tfrac{\pi_\theta(\tau\mid q)}{\pi_{ref}(\tau\mid q)}\big)\,\nabla_\theta \log \pi_\theta(\tau\mid q)\Big]
\;+\;
\mu\,\E_{\tau\sim \pi_\beta}\!\big[\nabla_\theta \log \pi_\theta(\tau\mid q)\big].
\]
In the estimator \eqref{eq:upge_form}, this corresponds to replacing the unified advantage by
\[
\widehat{A}_{uni}^{(\lambda)}(\tau,q)
=
r(\tau\mid q)
\;-\;\lambda \log \frac{\pi_\theta(\tau\mid q)}{\pi_{ref}(\tau\mid q)}
\;+\;
\mu\ \frac{\pi_\beta(\tau\mid q)}{\pi_\theta(\tau\mid q)}.
\]
All other expressions, including the masked estimator in \eqref{eq:upge_masked}, remain unchanged in form (with $\widehat{A}_{uni}$ replaced by $\widehat{A}_{uni}^{(\lambda)}$).

\subsection{PPO Clipping and the Stabilization Mask}
\label{app:ppo-mask-subsec}

With rollout policy $\pi_{\theta_{old}}$ and trust-region constraint $\mathrm{KL}(\pi_\theta\|\pi_{\theta_{old}})\le \delta$, the PPO surrogate
\[
\max_{\theta}\;\;
\E_{\tau\sim \pi_{\theta_{old}}}\Big[\min\big(r_\theta(\tau)\,A_{\theta_{old}}(\tau),\ \mathrm{clip}(r_\theta(\tau),1-\epsilon,1+\epsilon)\,A_{\theta_{old}}(\tau)\big)\Big],
\quad
r_\theta(\tau)=\frac{\pi_\theta(\tau)}{\pi_{\theta_{old}}(\tau)},
\]
has a piecewise derivative that is zero outside the trusted region in the harmful direction, yielding
\begin{equation}
\label{eq:ppo_mask_app}
\nabla_\theta
\approx
\E_{\tau\sim \pi_{\theta_{old}}}
\!\left[
\mathbb{1}_{stable}(\tau)\,
\frac{1}{\pi_{\theta_{old}}(\tau)}\,
A_{\theta_{old}}(\tau)\,
\nabla_\theta \pi_\theta(\tau)
\right],
\end{equation}
which matches the masked Unified Policy Gradient Estimator with $\pi_{ref}=\pi_{\theta_{old}}$ and $\widehat{A}=A_{\theta_{old}}$.

\end{document}